\documentclass[pdflatex,sn-mathphys]{sn-jnl}

\usepackage{graphicx}
\usepackage{amsmath, amssymb, amsfonts}
\usepackage{amsthm}
\usepackage{xcolor}
\usepackage{algorithm}
\usepackage{algorithmicx}
\usepackage{algpseudocode}

\usepackage{Sty/mcr}
\usepackage{bm}
\usepackage{subcaption}
\theoremstyle{definition} 
\newtheorem{lemma}{Lemma}
\algrenewcommand{\algorithmicrequire}{\textbf{Input:}}
\algrenewcommand{\algorithmicensure}{\textbf{Output:}}

\theoremstyle{thmstyleone}

\newtheorem{example}{Example}
\newtheorem{remark}{Remark}

\theoremstyle{thmstylethree}

\begin{document}

\title[Statistical Inference for Autoencoder-based Anomaly Detection after Representation Learning-based Domain Adaptation]{Statistical Inference for Autoencoder-based Anomaly Detection after Representation Learning-based Domain Adaptation}

\author[1,2]{Tran Tuan Kiet} 

\author[1,2]{Nguyen Thang Loi} 

\author*[1,2,3]{Vo Nguyen Le Duy}\email{duyvnl@uit.edu.vn}

\affil[1]{\orgname{University of Information Technology}, \city{Ho Chi Minh City}, \country{Vietnam}}

\affil[2]{\orgname{Vietnam National University}, \city{Ho Chi Minh City}, \country{Vietnam}}

\affil[3]{\orgname{RIKEN AIP}, \city{Tokyo}, \country{Japan}}

\abstract{
Anomaly detection (AD) plays a vital role across a wide range of domains, but its performance might deteriorate when applied to target domains with limited data. Domain Adaptation (DA) offers a solution by transferring knowledge from a related source domain with abundant data. However, this adaptation process can introduce additional uncertainty, making it difficult to draw statistically valid conclusions from AD results. 
In this paper, we propose STAND-DA---a novel framework for statistically rigorous Autoencoder-based AD after Representation Learning-based DA.
Built on the Selective Inference (SI) framework, STAND-DA computes valid $p$-values for detected anomalies and rigorously controls the false positive rate below a pre-specified level $\alpha$ (e.g., 0.05). 
To address the computational challenges of applying SI to deep learning models, we develop the GPU-accelerated SI implementation, significantly enhancing both scalability and runtime performance. This advancement makes SI practically feasible for modern, large-scale deep architectures. Extensive experiments on synthetic and real-world datasets validate the theoretical results and computational efficiency of the proposed STAND-DA method.

}


\maketitle

\section{Introduction}\label{sec1}

Anomaly detection (AD)---the task of identifying data points that deviate markedly from normal patterns---is a fundamental problem in both statistics and machine learning. Its importance is evident in the extensive body of research surveyed by \cite{aggarwal2016introduction}. 
The aim of AD is to detect rare or unexpected instances that may signal critical events, errors, or emerging phenomena. It plays a central role across various domains: in healthcare, it helps uncover abnormal clinical conditions \cite{wong2002rule, aggarwal2005abnormality}; in finance and cybersecurity, it enables the detection of fraudulent behavior \cite{pourhabibi2020fraud}; and in engineering, it supports damage detection \cite{avci2021review, du2020damage}.

\vspace{5pt}

AD might face significant challenges when applied to target domains with limited data—a situation frequently encountered in real-world settings. In such cases, models trained solely on the scarce target data may fail to capture meaningful patterns, leading to unreliable detection performance. One effective approach to address this issue is to incorporate additional information from a related source domain that contains abundant data. This strategy is known as Domain Adaptation (DA). 
DA aims to mitigate the distribution shift between the source and target domains by transferring knowledge in a way that improves model performance in the target domain. In the context of AD, this allows models to generalize better despite the lack of sufficient data in the target environment. By enriching the target dataset with informative structures learned from the source, DA can significantly enhance the robustness and accuracy of AD in low-data regimes.

\vspace{5pt}

One of the major challenges in AD is the risk of incorrectly identifying normal data points as anomalies, known as \emph{false positives}. 
These errors can have serious consequences, especially in high-stakes domains where decisions informed by AD may directly impact human lives or critical systems.
This issue might become even more pronounced in the context of DA, where discrepancies between the source and target domains can introduce additional uncertainty. 
For example, in medical applications, individuals with certain conditions in the source domain might be mapped to resemble healthy individuals in the target domain. This misalignment could lead to the erroneous classification of truly healthy patients as anomalous, potentially resulting in unnecessary or harmful interventions.
Given the implications of such mistakes, it is essential to develop inference methods that can rigorously control the false positive rate (FPR) and ensure the reliability of AD outcomes after DA.

\vspace{5pt}

In the context of AD following DA, it is equally crucial to manage the false negative rate (FNR), which is the probability of failing to identify actual anomalies. Statistical practice typically prioritize controlling the FPR at a predefined level of guarantee, e.g., $\alpha = 5\%$, while simultaneously striving to reduce the FNR, thereby improving the true positive rate (TPR). Adhering to this commonly used strategy, our work introduces a method that provides theoretical guarantees for properly controlling  the probability of incorrectly flagging normal data as anomalous, while also aiming to minimize the risk of overlooking true anomalies.

\vspace{5pt}

The work by \cite{le2024cad} is the first to introduce a statistical testing procedure for AD results following DA. 
By leveraging the Selective Inference (SI) framework \cite{lee2016exact}, their approach provides valid $p$-values for anomalies identified after adapting the source domain to the target, ensuring statistical validity in this context.
However, their work primarily addresses a simple setup that combines Optimal Transport (OT)-based DA \cite{flamary2016optimal} with a classical AD technique, namely Median Absolute Deviation (MAD).
While effective for this setting, the method does not generalize to modern deep learning (DL) models commonly used in both DA and AD tasks, which are significantly more complex than OT-based mappings and traditional detectors. 
As a result, applying their approach to DL-based DA and AD tasks remains infeasible.

\vspace{5pt}

In this paper, we carefully analyze the complex selection process associated with the intricate architecture of DL models used for DA and AD tasks. We then introduce a valid $p$-value by employing the SI framework to perform a rigorous statistical test on anomalies identified through this sophisticated procedure.
We note that, in this paper, we adopt the Representation Learning (RL)-based DA method introduced in \cite{shen2018wasserstein} as our representative DL-based DA approach, and the Autoencoder (AE)-based AD method as our DL-based AD technique, since these models are among the most widely used in both academia and industry.
The detailed discussions on extensions to other models are provided in \S\ref{sec:conclusion}.

\vspace{5pt}

\textbf{Contributions}. Our contributions are as follows:

\begin{itemize}

	\item We propose a novel statistical method, named \emph{STAND-DA} (\underline{St}atistical Inference for AE-based \underline{An}omaly \underline{D}etection after RL-based \underline{DA}), to test the results of AE-based AD in the context of RL-based DA. 
The proposed STAND-DA addresses the critical challenge of accounting for the impact of RL-based DA on downstream AE-based AD results, enabling statistically valid inference and providing well-calibrated \emph{p}-values for the detected anomalies. 
To the best of our knowledge, this is the first method that rigorously controls the FPR in the context of DL-based AD and DA.
	
	\vspace{5pt}
	
	\item We provide a GPU-accelerated implementation to enhance both the speed and practical applicability of SI when applied to complex DL models. Recently, the authors of \cite{katsuoka2025si4onnx} released Python packages for performing SI in DL models. However, their implementation runs only on CPUs, with the drawback of being computationally expensive and impractical for large-scale deep models---thus limiting the applicability of SI. To the best of our knowledge, this is the first publicly available Python library that supports GPU acceleration for SI, significantly improving its practicality for DL applications.
	
	\vspace{5pt}
	
	\item We conduct extensive experiments on both synthetic and real-world datasets to rigorously validate our theoretical results and computational efficiency, demonstrating the superior performance of the proposed STAND-DA method. The Python package is accessible via:

\begin{center}
	\url{https://github.com/DAIR-Group/STAND-DA} 
\end{center} 

\end{itemize}

\begin{example}
To highlight the importance of the proposed STAND-DA method, we present the illustrative example in Fig. \ref{fig:STAND-DA}.
The objective is to identify heart disease patients from Hospital T as anomalies in a scenario where only a small number of patient records are available. The source domain contains data from Hospital S, while the target domain corresponds to patients from Hospital T.
We begin by applying a RL-based DA technique to extract domain-invariant representations from both domains. Next, an AE-based AD model is used, marking the top 5$\%$ of patient records with the highest reconstruction errors as anomalies. 
This baseline approach, however, mistakenly identified two healthy individuals as anomalous cases.
To reduce such false results, we introduce an additional statistical inference step based on the proposed $p$-values, enabling us to identify both true and false detections. 
Furthermore, the experiment was repeated $N$ times, with the FPR results summarized in Table \ref{tbl:example}.
The results show that our method consistently maintained the FPR below the significance threshold of $\alpha = 0.05$, whereas competing methods failed to achieve this control.

\end{example}

\begin{figure*}[!t]
\centering
\includegraphics[width=\textwidth]{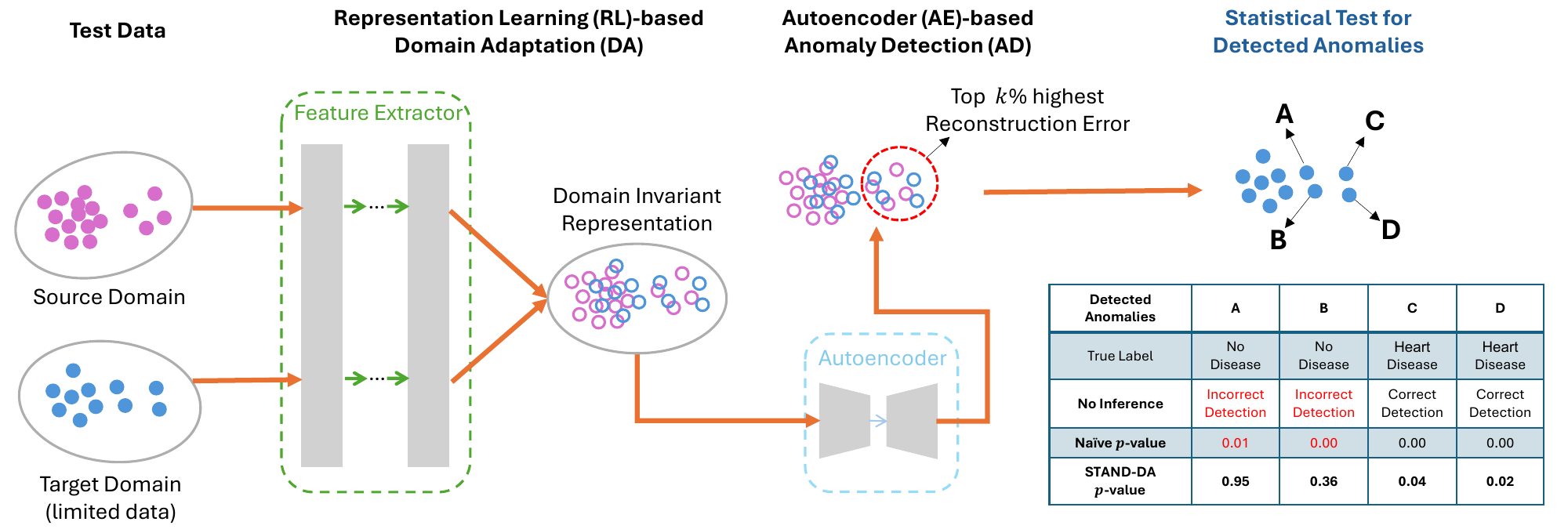}

\vspace{15pt}
\caption{Illustration of the proposed STAND-DA method. 
Performing AD-DA without the inference step leads to incorrect anomaly detections, such as cases {\bf A} and {\bf B}. In this setting, naïvely computed $p$-values remain small even for false detections. In contrast, the proposed STAND-DA method correctly identify both false positives (FPs) and true positives (TPs), yielding large $p$-values for FPs and small $p$-values for TPs.
}
    \label{fig:STAND-DA}
\vspace{-10pt}
\end{figure*}

\begin{table}[!t]
\renewcommand{\arraystretch}{1.3}
\centering
\caption{The key strength of the proposed STAND-DA lies in its ability to control the False Positive Rate (FPR).}
\label{tbl:example}
\begin{tabular}{|p{2cm}|p{2cm}|p{2cm}|p{3cm}|}
\hline
& \textbf{No Inference} & \textbf{Naive} & \textbf{STAND-DA} \\
\hline
  $N = 120$ & FPR = 1.0 & 0.45 & \textbf{0.050} \\
   \hline
  $N = 240$ & FPR = 1.0 & 0.70 & \textbf{0.052} \\
  \hline
\end{tabular}
\vspace{-10pt}
\end{table}

\textbf{Related works}. 
While AD has been extensively studied in the literature \cite{aggarwal2016introduction}, relatively little attention has been given to evaluating AD outcomes through a hypothesis testing framework.
Some earlier works, such as \cite{srivastava1998outliers} and \cite{pan1995multiple}, explore the use of likelihood ratio tests within the mean-shift model to assess whether a specific data point deviates significantly from the norm---effectively treating it as an outlier.
However, these traditional statistical approaches are only valid when the anomalies under consideration are specified a priori.
If one instead applies these methods post hoc---i.e., after anomalies have been identified by an AD algorithm---the statistical guarantee breaks down, and crucially, the FPR can no longer be reliably controlled.

\vspace{5pt}

To maintain control over the FPR when applying traditional statistical tests, it is crucial to adjust for multiple comparisons.
A widely adopted approach is the Bonferroni correction, which compensates for the number of hypotheses by applying a stringent threshold that increases with the dataset size.
In particular, the adjustment factor can reach up to $2^n$, where $n$ represents the number of data points.
This exponential scaling quickly renders the correction overly strict, making it impractical for larger datasets and resulting in excessively conservative conclusions.

\vspace{5pt}

SI has gained attention as an effective strategy for addressing the shortcomings of traditional statistical testing---particularly its reliance on overly conservative multiple testing corrections.
Rather than considering the exponentially increasing value of $2^n$, SI focuses on the specific outcome produced by the detection algorithm.
By conditioning the inference on this observed set of anomalies, the correction factor is effectively reduced to 1, enabling valid statistical conclusions without sacrificing power.
This idea of conditioning on the selection event forms the foundation of the conditional SI framework, originally introduced in the seminar work of \cite{lee2016exact}, where it was applied to inference on features selected by the Lasso.

\vspace{5pt}

Building on its foundational role in enabling valid inference after feature selection, the seminal work of \cite{lee2016exact} has catalyzed a broad wave of research extending SI to a wide range of learning settings. 
In the supervised domain, SI have been adapted to accommodate increasingly complex models---including boosting~\cite{rugamer2020inference}, decision trees~\cite{neufeld2022tree}, kernel-based methods~\cite{yamada2018post}, higher-order interaction models~\cite{suzumura2017selective,das2021fast}, and DL~\cite{duy2022quantifying, miwa2023valid}. This body of work has expanded the applicability of SI well beyond its original focus on linear models and Lasso-based variable selection~\cite{loftus2014significance, fithian2015selective, tibshirani2016exact, yang2016selective, hyun2018exact, sugiyama2021more, fithian2014optimal, duy2021more}.
Importantly, the versatility of SI also extends to unsupervised learning tasks, where it has been applied to problems such as change point detection~\cite{umezu2017selective, hyun2018post, duy2020computing, sugiyama2021valid, jewell2022testing}, clustering~\cite{lee2015evaluating, inoue2017post, gao2022selective, chen2022selective}, and segmentation~\cite{tanizaki2020computing, duy2022quantifying}. Beyond model-based inference, SI has even been employed for hypothesis testing on complex distance metrics, including Dynamic Time Warping (DTW) \cite{duy2022exact} and Wasserstein distances \cite{duy2021exact}.

\vspace{5pt}

The work most closely aligned with ours is that of \cite{le2024cad}, which marks the first attempt to establish a formal hypothesis testing framework for AD following DA.
Drawing on the SI framework, their approach provides statistically valid $p$-values for anomalies detected after aligning source and target domains—thus ensuring reliable inference in this setting.
However, their focus is limited to a relatively simple scenario, where domain adaptation is performed using OT methods \cite{flamary2016optimal} and anomalies are identified using the classical MAD technique.
While their method is effective within this narrow context, it lacks the flexibility to handle the complexities of modern DL-based approaches to both DA and AD. These methods often involve sophisticated processes, making the original SI-based procedure inapplicable.
Consequently, extending their method to DL-driven pipelines remains an open challenge.

\section{Problem Statement}\label{sec2}

Let us consider two random matrices $X^s \in \mathbb{R}^{n_s \times d}$ and $X^t \in \mathbb{R}^{n_t \times d}$ defined as follows:
\begin{align*}
	X^s &= 
	\big ( 
	{\bm X^s_1}^\top,
	{\bm X^s_2}^\top,
	\dots,
	{\bm X^s_{n_s}}^\top
	\big )^\top
	= 
	M^s + \cE^s,
	\quad \cE^s \sim \MM \NN_{n_s \times d} (0, U^s, V^s),
	\\
	X^t &= 
	\big ( 
	{\bm X^t_1}^\top,
	{\bm X^t_2}^\top,
	\dots,
	{\bm X^t_{n_t}}^\top
	\big )^\top
	= 
	M^t + \cE^t,
	\quad \cE^t \sim \MM \NN_{n_t \times d} (0, U^t, V^t),
\end{align*}
where $n_s$ and $n_t$ are the number of instances in the source and target domains, $\bm{X}^s_i$ is a vector in $\mathbb{R}^d$ for each $i \in [n_s] = \{1, 2, \dots, n_s\}$, and $\bm{X}^t_j \in \RR^d$ for any $j \in [n_t]$.
The matrices $M^s \in \RR^{n_s \times d}$ and $M^t \in \RR^{n_t \times d}$ are unknown signal matrices.
The $\cE^s$ and $\cE^t$ are noise matrices that follow \emph{matrix normal distributions}, where $U^s$ and $U^t$ are the row covariance matrices, and $V^s$ and $V^t$ are the column covariance matrices.
The $U^s$, $U^t$, $V^s$, and $V^t$ are assumed to be known or estimable from independent data.
In this paper, we consider the setting in which the number of instances in the target domains is limited, i.e., $n_t$ is much smaller than $n_s$.
The goal is to conduct a statistical test on the results of AE-based AD after RL-based DA.

Furthermore, because the notation of matrix normal distributions might be complicated, we will use the following equivalent vectorized forms:
\begin{align*}
    {\rm vec}(X^s) &\sim \NN({\rm vec}(M^s), \Sigma^s),
    \\
    {\rm vec}(X^t) &\sim \NN({\rm vec}(M^t), \Sigma^t),
\end{align*}
where ${\rm vec}(.)$ denotes the vectorization operator that transforms a matrix into a vector with concatenated rows, $\otimes$ denotes the Kronecker product, $\Sigma^s = U^s \otimes V^s$, and $\Sigma^t = U^t \otimes V^t$.
This equivalence is a direct consequence of the standard definition of the matrix normal distribution \cite{gupta2018matrix}.

\vspace{5pt}

We would like to note that our primary objective is to perform statistical inference in the testing phase, not the training phase. 
This distinction is crucial, as we aim to evaluate the reliability of the outcomes produced by models that have already been trained. 
Accordingly, the datasets $X^s$ and $X^t$ serve as test data. 
These datasets are passed through a pipeline consisting of two already trained models: a RL-based DA model and an AE-based AD model.
Our objective is to rigorously assess the statistical reliability of the AD results produced by this pipeline on $X^s$ and $X^t$---that is, whether they are reliable---taking into account the models' selective behavior and the fact that no further learning or parameter updates are performed during the testing phase.

\subsection{Wasserstein Distance Guided Representation Learning for Domain Adaptation (RL-based DA) \cite{shen2018wasserstein}} 

\textbf{Training phase.} In RL-based DA's learning approach \cite{shen2018wasserstein}, a feature extractor---typically implemented as a neural network---is employed to learn domain invariant representations from both the source and target domains. 
Given an instance $\bm X \in \RR^d$ from either domain, the feature extractor learns a mapping $f_{\rm extractor}: \mathbb{R}^d \mapsto \mathbb{R}^{d^\prime}$ that projects the input into a $d^\prime$-dimensional representation.
Then, to minimize the discrepancy between the source and target  distributions, a domain critic is used, as suggested in \cite{arjovsky2017wasserstein}, with the objective of estimating the Wasserstein distance between the two distributions.

\vspace{5pt}

\textbf{Testing phase.}
In the testing phase, the trained feature extractor $f_{\rm extractor}$ is applied to the test data to obtain the domain invariant representations $\tilde{X}^s$ and $\tilde{X}^t$:
\begin{align} \label{eq:domain_invariant_features}
	\tilde{X}^s = f_{\rm extractor} (X^s),
	\quad 
	\tilde{X}^t = f_{\rm extractor} (X^t).
\end{align}
These representations are then used as input for the downstream AE-based AD task.

\subsection{Autoencoder-based Anomaly Detection after RL-based DA}

Autoencoder (AE)-based anomaly detection leverages neural networks trained to reconstruct normal data patterns by compressing inputs into a lower-dimensional representation and then decoding them back.
During training, the AE learns to reconstruct normal inputs by minimizing the reconstruction error. 
Since the model captures the underlying structure of normal patterns, it typically fails to accurately reconstruct anomalous inputs, leading to higher reconstruction errors.
These errors can then serve as anomaly scores, with larger values indicating potential outliers.
A common approach for identifying anomalies is to consider a predefined percentage of instances with the highest reconstruction errors as anomalous.
In this paper, we consider the top 5\% of instances with the highest reconstruction errors as anomalies.

\vspace{5pt}

Given the trained AE-based AD model, we apply it to the domain invariant feature representations of both the source and target domains, as defined in \eq{eq:domain_invariant_features}, to obtain a set $\cO$ containing the indices of anomalies in the target domain:
\begin{equation} \label{eq:detected_anomalies}
        \mathcal{A}: 
        \big \{ 
        \tilde{X}^s, \tilde{X}^t
        \big\} \mapsto \mathcal{O} \subset [n_t],
\end{equation}
where $\cA$ denotes the AD algorithm.
In this paper, we primarily focus on AE-based AD using mean absolute error (MAE) as the reconstruction error, as it is one of the most commonly used methods in both academia and industry. However, our approach can be extended to other AD models and alternative reconstruction error as discussed in \S\ref{sec:conclusion}.

\subsection{Statistical inference and decision making with a $p$-value}
To rigorously assess the statistical significance of the identified anomalies, we formulate and test the following null and alternative hypotheses for each $j \in \cO$:
\begin{align*}
	{\rm H}_{0, j}:
	M^t_{j, k} = \bar{M}^t_{\cO^c, k}, ~\forall k \in [d]
	\quad \text{vs.} \quad
	{\rm H}_{1, j}:
	\exists k \in [d], ~M^t_{j, k} \neq \bar{M}^t_{\cO^c, k}
\end{align*}
where 
\begin{align*}
    \bar{M}^t_{\cO^c, k} = \frac{1}{n_t - |\mathcal{O}|} \sum \limits_{ \ell \in [n_t] \setminus \mathcal{O}} X_{\ell, k}^t,
\end{align*}
the convention $\cO^c$ means ``$\cO$'s complement". In other words, our goal is to test whether each detected anomaly $j \in \cO$ truly deviates from the remaining data points after excluding the anomaly set $\cO$ itself.
To test the hypotheses, the test statistic is defined as:
\begin{equation}
\label{eq:test_statistic}
        T_j =
        \sum \limits_{k \in [d]} 
        \left | 
        X^t_{j, k} - \bar{X}^t_{j, k}
        \right |
	= 
	\bm \eta_j^\top {\rm vec} 
	\begin{pmatrix} X^s \\ X^t \end{pmatrix},
\end{equation} 
where
\begin{equation}
\label{eq:etaj}
        \boldsymbol{\eta}_j = {\rm vec}
        \begin{pmatrix}
        0^s \\
        \boldsymbol{s}_j^T \odot \left (E_j^t - \tfrac{1}{n_t - |\mathcal{O}|} E_{\mathcal{O}^c}^t \right )
        \end{pmatrix}.
\end{equation}
Here, 
$0^s \in \mathbb{R}^{n_s \times d}$ denotes a matrix with all elements set to zero. 
The matrix $E_j^t \in \mathbb{R}^{n_t \times d}$ has all entries in the the $j^\text{th}$ row set to 1 and all other entries set to 0.
Similarly, $E_{\mathcal{O}^c}^t \in \mathbb{R}^{n_t \times d}$  is defined such that rows corresponding to the indices in $\cO$ are set to 0, and all other rows are set to 1.
The vector $\boldsymbol{s}_j \in \mathbb{R}^{d}$ are the signs of the component-wise subtractions in (\ref{eq:test_statistic}).
Finally, the operator $\odot$ denotes the column-wise product, where each column of $\boldsymbol{s}_j^\top$ is multiplied element-wise with the corresponding column of the matrix it is applied to. 

\vspace{5pt}

Once the test statistic in (\ref{eq:test_statistic}) is computed, the next step is to compute the corresponding $p$-value. 
For a given significance level $\alpha \in [0, 1]$ (e.g., 0.05), we reject the null hypothesis $\text{H}_{0,j}$ and determine $X_j^t$ as an anomaly if the $p$-value is less than or equal to $\alpha$. Otherwise, if the $p$-value exceeds $\alpha$, we conclude that the evidence is insufficient to conclude that $X_j^t$ is an anomaly.

\subsection{Challenge of computing a valid $p$-value}
The traditional (naive) $p$-value, which does not properly account for the effects of RL-based DA and AE-based AD, is defined as:
\begin{align*}
	p_j^{\text{naive}} = \mathbb{P}_{\text{H}_{0,j}}
	\Bigg( 
	\left| \boldsymbol{\eta}_j^\top {\rm vec} \begin{pmatrix} X^s \\ X^t \end{pmatrix} \right| \geq \left| \boldsymbol{\eta}_j^\top {\rm vec} \begin{pmatrix} X^s_{\rm obs} \\ X^t_{\text{obs}} \end{pmatrix} \right| 
	\Bigg),
\end{align*}
where $X^s_{\text{obs}}$ and $X^t_{\text{obs}}$ are the observations (realizations) of the random matrices $X^s$ and $X^t$, respectively. If the vector $\boldsymbol{\eta}_j$ is independent of the RL-based DA and AE-based AD algorithms, the naive $p$-value is valid in the sense that
\begin{equation}
    \mathbb{P}
    \Big(\underbrace{
    p_j^{\text{naive}} \leq \alpha \mid \text{H}_{0,j} \text{ is true}
    }_{\text{a false positive}}
    \Big)
    = \alpha, \quad \forall \alpha \in [0,1].
\end{equation}
That is, the probability of obtaining a false positive is controlled under a specified level of confidence. However, in our setting, the vector $\boldsymbol{\eta}_j$ inherently depends on both the DA and AD processes. As a result, the validity condition of the $p$-value in (4) no longer holds, rendering the naive $p$-value \textit{invalid}.

\section{Proposed Method} \label{sec:proposed_method}

This section details our proposed STAND-DA method for computing valid $p$-values for the detected anomalies.

\subsection{The valid $p$-value in STAND-DA} \label{subsec:valid_p_value}

To compute valid $p$-values for the detected anomalies, we need to characterize the distribution of the test statistic defined in \eqref{eq:test_statistic}. To this end, we utilize the concept of SI, i.e., we consider the distribution of the test statistic conditional on both the AE-based AD results after RL-based DA and the signs of the subtractions involved in \eqref{eq:test_statistic}.
\begin{align} \label{eq:conditional_distribution}
	\mathbb{P} \Bigg ( 
	\bm \eta_j^\top {\rm vec}{X^s \choose X^t }
	~
	\Big |
	~ 
	\cO_{X^s, X^t}
	=
	\cO_{\rm obs},
    ~
    \cS_{ X^s,  X^t}
    =
    \cS_{\rm obs}
	\Bigg ),
\end{align}
where $\cO_{X^s, X^t}$ and $\cS_{X^s, X^t}$ respectively denote the AD results after DA and the set of all signs of the subtractions in \eqref{eq:test_statistic} \emph{for any random} matrices $X^s$ and $X^t$, $\cO_{\rm obs} = \cO_{X^s_{\rm obs}, X^t_{\rm obs}}$, and $\cS_{\rm obs} = \cS_{X^s_{\rm obs}, X^t_{\rm obs}}$.
Based on the distribution in \eqref{eq:conditional_distribution}, we define the selective $p$-value as follows:
\begin{align} \label{eq:selective_p}
	p^{\rm selective}_j = 
	\mathbb{P}_{\rm H_{0, j}} 
	\Bigg ( 
		\left | \bm \eta_j^\top {\rm vec}{X^s \choose X^t } \right |
		\geq 
		\left | \bm \eta_j^\top {\rm vec}{X^s_{\rm obs} \choose X^t_{\rm obs} } \right |
		~
		\Bigg | 
		~
		\cC
	\Bigg ), 
\end{align}
where $\cC$ denotes the conditioning event, defined as:
\begin{align} \label{eq:condition_event}
    \cC = \Bigg \{ 
    \cO_{X^s, X^t}
    =
    \cO_{\rm obs}, ~
    \cS_{X^s, X^t}
    =
    \cS_{\rm obs}, ~
    \cQ_{X^s, X^t}
    =
    \cQ_{\rm obs}
    \Bigg \}.
\end{align}
Here, $\cQ_{X^s, X^t}$ denotes the \emph{nuisance component}, defined as:
\begin{align} \label{eq:nuisance_component}
    \cQ_{X^s, X^t} = 
    \big (
    I_{n_s + n_t}
    - 
    \bm b
    \bm \eta_j^\top
    \big )
    ~ 
    {\rm vec}{X^s \choose X^t },
\end{align}
where $\bm b = \tfrac{\Sigma \bm \eta_j}{\bm \eta_j^\top \Sigma \bm \eta_j}$ with 
$\Sigma = 
\begin{pmatrix}
\Sigma^s & 0 \\
0 & \Sigma^t
\end{pmatrix}
$.

\begin{remark}
The nuisance component $\cQ_{X^s, X^t}$ corresponds to the component $\bm z$ introduced in the the seminal paper of \cite{lee2016exact} (see Sec. 5, Eq. (5.2), and Theorem 5.2). We note that 
conditioning additionally on the nuisance component---though required for a technical reason---is a common practice in the conditional SI literature and is employed in in almost all SI-related studies we reference.
\end{remark}

\begin{lemma}
    \label{lem:valid_p_value}
    The selective $p$-value in \eqref{eq:selective_p} satisfies the property of a valid $p$-value:
    \begin{align*} 
        \mathbb{P}_{H_{0,j}} \left(
        p_j^{\text{selective}} \leq \alpha
        \right)
        =
        \alpha,
        \quad \forall \alpha \in [0,1].
    \end{align*}
    \begin{proof}
        The proof is deffered to the \textit{Appendix} \ref{appx:proof_valid_selective_p}.
    \end{proof}
\end{lemma}
Lemma \ref{lem:valid_p_value} shows that the selective $p$-value defined in \eqref{eq:selective_p} allows us to control the FPR at the significant level $\alpha$. To calculate this selective $p$-value, we must characterize the conditioning event $\cC$ in \eqref{eq:condition_event}, which will be detailed in the following section.

\subsection{Characterization of the Conditioning Event $\cC$ in \eqref{eq:condition_event}} \label{subsec:characterization_conditioning_event}
We define the set of ${\rm vec}{X^s \choose X^t} \in \RR^{(n_s + n_t) d}$ that satisfies the conditions in \eqref{eq:condition_event} as:
\begin{align} \label{eq:condition_event_set}
    \mathcal{D} 
    = 
    \left \{
    {\rm vec}{X^s \choose X^t} ~
    \Big | ~
    \cO_{X^s, X^t}
    =
    \cO_{\rm obs}, ~ 
    \cS_{X^s, X^t}
    =
    \cS_{\rm obs}, ~ 
    \cQ_{X^s, X^t}
    =
    \cQ_{\rm obs}
    \right \}.
\end{align}
As stated in the following lemma, the conditional data space $\mathcal{D}$ corresponds to a line in $\mathbb{R}^{(n_s + n_t)d}$.
\begin{lemma} \label{lem:data_line}
    The set $\mathcal{D}$ in \eqref{eq:condition_event_set} can be re-parameterized using a scalar parameter $z \in \mathbb{R}$ as follows:
    \begin{align} \label{eq:condition_event_set_line}
        \mathcal{D} 
        = 
        \left \{
        {\rm vec}{X^s \choose X^t}
        =
        \bm a + \bm b z ~
        \Big | ~
        z \in \cZ
        \right \},
    \end{align}
    where $\bm a = \cQ_{\rm obs}$, $\bm b$ is defined in \eqref{eq:nuisance_component}, and
    \begin{align} \label{eq:cZ}
        \cZ = 
        \big \{
        z \in \mathbb{R}
        \mid 
        \cO_{\bm a + \bm b z}
        =
        \cO_{\rm obs}, ~
        \cS_{\bm a + \bm b z}
        =
        \cS_{\rm obs}
        \big \}.
    \end{align}
    Here, by a slight abuse of notation, $\cO_{\bm a + \bm b z}$ is equivalent to $\cO_{X^s, X^t}$. This equivalence similarly holds for $\cS_{\bm a + \bm b z}$.
    \begin{proof}
        The proof is deferred to the Appendix \ref{appx:proof_data_line}.
    \end{proof}
\end{lemma}

\begin{remark}
Lemma \ref{lem:data_line} demonstrates that it is unnecessary to work in the full $(n_s + n_t)d$-dimensional data space. Instead, we can project the data onto a one-dimensional space $\cZ$ as defined in \eqref{eq:cZ}. This reduction to a line has been implicitly utilized in \cite{lee2016exact} and explicitly addressed in Sec. 5 of \cite{liu2018more}.
\end{remark}

\textbf{Reformulation of the selective $p$-value in \eqref{eq:selective_p}.}
Consider a random variable along with its observed value:
\begin{align} \label{eq:random_variable}
    Z = 
    \bm \eta_j^\top {\rm vec}{X^s \choose X^t }
    \in 
    \mathbb{R}
    \quad
    \text{and}
    \quad 
    Z_{\rm obs} = 
    \bm \eta_j^\top {\rm vec}{X^s_{\rm obs} \choose X^t_{\rm obs} }
    \in 
    \mathbb{R},
\end{align}
the selective $p$-value in \eqref{eq:selective_p} can then be rewritten as:
\begin{align} \label{eq:valid_p_value_reformulated}
    p_j^{\text{selective}} = 
    \mathbb{P}_{\rm H_{0,j}} \Big ( 
    \left | Z \right |
    \geq
    \left | Z_{\rm obs} \right |
    \mid
    Z \in \cZ
    \Big ).
\end{align}
Once the truncation region $\cZ$ is known, calculating the selective $p$-value in \eqref{eq:valid_p_value_reformulated} becomes straightforward. 
Consequently, the remaining crucial task is to identify the truncation region $\cZ$.

\begin{figure*}[!t]
    \centering
    \includegraphics[width=.9\textwidth]{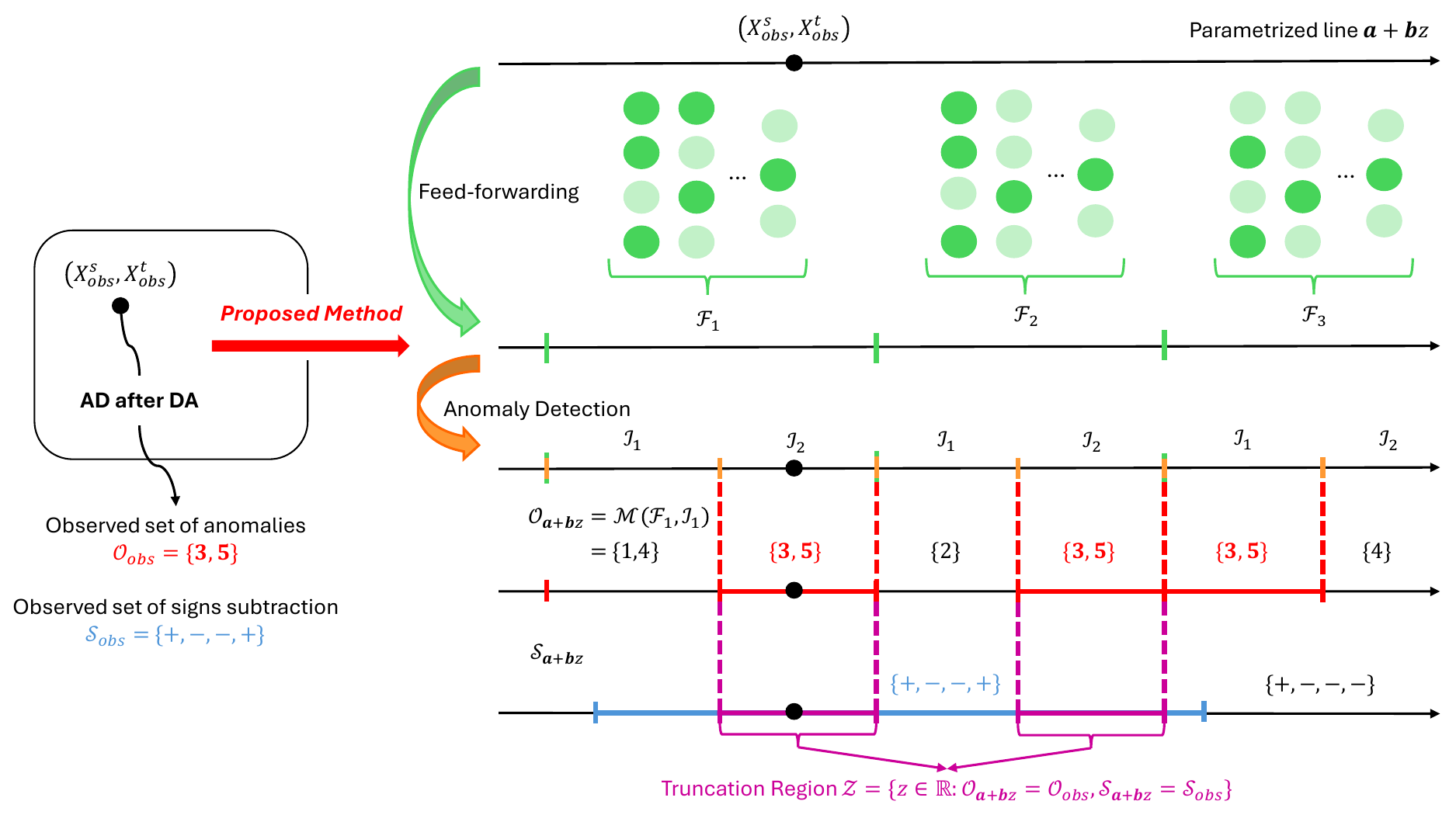}
    \vspace{1pt}
    \caption{
    A schematic illustration of our proposed STAND-DA. 
    We begin by performing RL-based DA, followed by AD through an AE model.
    Next, we parametrize the data using a scalar parameter $z$ in the dimension of the test statistic to define the truncation region $\cZ$.
        This region consists of all data points that yield the \emph{same} AD results $\cO$ and share the \emph{same} sign-subtraction pattern $\cS$ as the observed data. 
            Finally, statistical inference is performed by conditioning on $\cZ$. To improve computational efficiency, we adopt a divide-and-conquer approach for efficiently identifying the $\cZ$.
     }
    \label{fig:line_search_approach}
    \vspace{-10pt}
\end{figure*}
    
\subsection{Indentification of Truncation Region $\cZ$} \label{subsec:truncation_region}

To compute the selective $p$-value, as discussed in the previous section, we must identify the truncation region $\cZ$. However, this region cannot be determined directly. To address this, we adopt a divide-and-conquer  (illustrated in Fig. \ref{fig:line_search_approach})---drawing inspiration from \cite{duy2022more}---to systematically identify $\cZ$ through the following steps:

\begin{itemize}
    \item 
    We decompose the overall problem into a series of sub-problems. These are conditioned not only on the detected anomalies and the signs of the subtractions but also on the activations of all activation functions within the feature extractor architecture in the RL-based DA and the AE, as well as all subsequent processing steps to obtain the final AD results.
  
    \item We prove that solving each sub-problem is computationally efficient.
    \item 
	We construct the truncation region $\cZ$ through the integration of solutions from multiple sub-problems.
\end{itemize}

\vspace{5pt}

\textbf{Divide-and-conquer strategy.}
To facilitate further decomposition, let us rewrite the truncation region $\cZ$ in \eqref{eq:cZ} as:
\begin{align*}
    \cZ
    =
    \big \{
        z \in \mathbb{R}
        \mid 
        \cO_{\bm a + \bm b z} = \cO_{\rm obs},
        \cS_{\bm a + \bm b z} = \cS_{\rm obs}
    \big \} 
    =
    \cZ_1 
    \cap 
    \cZ_2,
\end{align*}
where 
\begin{align} 
    \cZ_1
    = 
    \big \{
        z \in \mathbb{R}
        \mid
        \cO _{\bm a + \bm b z} = \cO_{\rm obs}
    \big \} \quad \text{and} \quad  
    \cZ_2 
    = 
    \big \{
        z \in \mathbb{R}
        \mid
        \cS _{\bm a + \bm b z} = \cS_{\rm obs}
    \big \} \label{eq:cZ_1_2}.
\end{align}
After the decomposition, we can identify the truncation regions $\cZ_1$ and $\cZ_2$ separately and combine them to obtain the final truncation region $\cZ$.


\vspace{5pt}

To identify the truncation region $\cZ_1$, let us denote by $U$ the number of all possible states of the activation functions within the architecture of the feature extractor in the RL-based DA and the AE-based AD.
For each $u \in [U]$, let $V^u$ represent the number of all possible sequences of steps to identify the anomalies based on the reconstruction error of the AE.
The entire one-dimensional data space $\mathbb{R}$ can be decomposed as:
\begin{align*}
    \mathbb{R} 
    = 
    \bigcup \limits_{u \in [U]}
    \bigcup \limits_{v \in [V^u]}
    \underbrace{
    \left \{
    z \in \mathbb{R}
    ~
    \Bigg | 
    ~
    \begin{array}{l}
        \cF_{\bm a + \bm b z} = \cF_{u}, \\
        \cI_{\bm a + \bm b z} = \cI_{v} \\
    \end{array}
    \right \}
    }_{\text{A sub-problem of additional conditioning}},
\end{align*}
where $\cF_{\bm a + \bm b z}$ denotes the activation pattern of the activation functions during the feed-forward process through the feature extractor of the RL-based DA and the AE, and $\cI_{\bm a + \bm b z}$ represents the sequence of steps used to identify anomalies based on the reconstruction error of the AE.
Our goal is to find a set: 
\begin{align} \label{eq:cV}
    \cV = 
    \big \{
    (u, v)
    \mid
    \cM(\cF_{u}, \cI_{v})
    = 
    \cO_{\rm obs}
    \big \},
\end{align}
for all $u \in [U]$ and $v \in [V^u]$, the function $\cM$ is defined as:
\begin{align} \label{eq:cM}
    \cM : 
    (\cF_{\bm a + \bm b z}, \cI_{\bm a + \bm b z})
    \mapsto
    \cO_{\bm a + \bm b z}.
\end{align}
Afterward, we can obtain the $\cZ_1$ as:
\begin{align} \label{eq:z1_union}
    \cZ_1 
    = 
    \left \{
    z \in \mathbb{R}
    ~
    |
    ~
    \cO_{\bm a + \bm b z}
    =
    \cO_{\rm obs}
    \right \} 
    =
    \bigcup \limits_{(u, v) \in \cV}
    \big \{
    z \in \mathbb{R}
    \mid
    \cF_{\bm a + \bm b z} = \cF_{u},
    \cI_{\bm a + \bm b z} = \cI_{v}
    \big \}.
\end{align}
Finally, the truncation region $\cZ$ can be obtained as:
\begin{align} \label{eq:z_final}
    \cZ 
    &= 
    \cZ_1 
    \cap 
    \cZ_2 \nonumber \\
    &=
    \Big (\bigcup \limits_{(u, v) \in \cV}
    \big \{
    \underbrace{
    z \in \mathbb{R}
    \mid 
    \cF_{\bm a + \bm b z} = \cF_{u},
    \cI_{\bm a + \bm b z} = \cI_{v}
    }_{\text{a sub-problem}}
    \big \} \Big )
    \bigcap 
    \big \{
    \underbrace{
    z \in \mathbb{R}
    \mid
    \cS_{\bm a + \bm b z}
    =
    \cS_{\rm obs}
    }_{\text{a sub-problem}}
    \big \}.
\end{align}


\vspace{5pt}
\textbf{Solving each sub-problem.}
For any $u \in [U]$ and $v \in [V^u]$, we define the subset of the one-dimensional projection of the dataset onto a line corresponding to the sub-problem in $\cZ_1$ as follows:
\begin{align} \label{eq:extra_conditioning}
    \cZ_{u, v}
    = 
    \big \{
    z \in \mathbb{R}
    ~
    |
    ~
    \cF_{\bm a + \bm b z} = \cF_{u},
    \cI_{\bm a + \bm b z} = \cI_{v}
    \big \}.
\end{align}
The $\cZ_{u, v}$ can be re-written as $\cZ_{u, v}  = \cZ_u \cap \cZ_v$ in which 
\begin{align*}
	\cZ_u  = 
	\big \{ 
	z \in \RR
	\mid 
	\cF_{\bm a + \bm b z} &= \cF_u
	\big \} 
	\quad
	\text{and}
	\quad
	\cZ_v = 
	\big \{ 
	z \in \RR
	\mid 
	\cI_{\bm a + \bm b z} = \cI_v
	\big \}.
\end{align*}

\vspace{5pt}

\begin{lemma} \label{lem:activation_function}

Consider the broad class of piecewise-linear neural networks. The set $\mathcal{Z}_u$ can be characterized by a collection of linear inequalities with respect to $z$:
\[
\mathcal{Z}_u = \left\{ z \in \mathbb{R} \;\middle|\; \bm{p}z \leq \bm{q} \right\},
\]
where the vectors $\bm{p}$ and $\bm{q}$ are defined are defined in the following proof.
\end{lemma}

\begin{proof}
Consider a deep neural network with \(L\) layers. In each layer \(l\), a linear transformation is followed by an element-wise ReLU activation. Denoting by \(d_l\) the number of units in layer \(l\), the layer output is
\begin{align} \label{eq:X_l}
X_{(l)} = \operatorname{ReLU} \left(X_{(l-1)} W_{(l)} + \mathbf{1}_n\,\mathbf{c}_{(l)}^\top \right),
\end{align}
where \( X_{(0)} = \begin{pmatrix} X^s \\ X^t \end{pmatrix} \), \( X_{(l)} \in \mathbb{R}^{n \times d_l} \) is the output to the \( l \)-th layer, \( W_{(l)} \in \mathbb{R}^{d_{l-1} \times d_l} \) is the weight matrix, and \( \bm{c}_{(l)} \in \mathbb{R}^{d_l} \) is the bias vector. The vector \( \bm{1}_n \in \mathbb{R}^n \) is an \( n \)-dimensional column vector of ones, and the ReLU function is applied element-wise.
The activation pattern on the $l$-th layer can be characterized by the following set:
\begin{align}
\cF_{(l)} =  \text{sign}\left( \bm{X}_{(l-1)} \mathcal{W}_{(l)} + \bm{1}_n \bm c_{(l)}^\top \right).
\label{F_l}
\end{align}
Then, the entire activation patterns induced by the forward propagation process can be represented as
\begin{align*}
        \mathcal{F}_u = \bigcap_{l=1}^{L} \mathcal{F}_{(l)}.
\end{align*}
Regarding the selection event inequalities associated with $\mathcal{F}_{(l)}$ w.r.t. $
\begin{pmatrix} X^s \\ X^t \end{pmatrix} = A + Bz,
$
where $A = \operatorname{matrix}(\bm{a})$ and $B = \operatorname{matrix}(\bm{b})$, and $\operatorname{matrix}(\cdot)$ denotes the operator that reshapes a vector into a matrix of size $(n_s + n_t) \times d$,  
we can express as:
\begin{align}
        & F_{(l)} \circ \left(B_{(l-1)} W_{(l)} z + A_{(l-1)}W_{(l)}  + \mathbf{1}_n \bm c_{(l)}^\top \right) \geq 0
        \label{F_l_ineq}
        \\
        \Leftrightarrow\ 
        & F_{(l)} \circ \left ( - B_{(l-1)} W_{(l)}  z \right ) \leq F_{(l)} \circ \left(  A_{(l-1)} W_{(l)} + \mathbf{1}_n \bm c_l^\top \right). \nonumber
\end{align}
Here, the operator $\circ$ is element-wise product, $A_{(l)} \in \mathbb{R}^{(n_s + n_t) \times d_{l}}$ and $B_{(l)} \in \mathbb{R}^{(n_s + n_t) \times d_{l}}$, which are associated with the data $A_{(l)} + B_{(l)} z$ after being propagated through the $l$-th layer, can be computed iteratively as follows:
    \begin{align}
       \left [A_{(l)} \right ]_{ij} &= 
        \begin{cases}
            \left [A_{(l-1)} {W}_{(l)} + \mathbf{1}_n \bm c_{(l)}^\top \right ]_{ij},  
            &\text{if } \left [X_{(l)} \right ]_{ij} > 0, \\
            0, & \text{otherwise.}
        \end{cases}
        \label{A_l}
    \\
        \left [B_{(l)} \right ]_{ij} &= 
        \begin{cases}
            \left [B_{(l-1)} {W}_{(l)} \right ]_{ij}, & 
            \text{if } 
            \left [X_{(l)} \right ]_{ij} > 0, \\
            0, & \text{otherwise.}
        \end{cases}
        \label{B_l}
    \end{align}
    for all $i \in [n_s+n_t]$ and $j \in [d_l]$.
    Finally, the vectors $\bm{p}$ and $\bm{q}$ can be obtained as follows:
    \begin{align*}
        \bm p
        =
        \begin{pmatrix}
            \bm p_{(1)} \\
            \bm p_{(2)} \\
            \vdots \\
            \bm p_{(L)}
        \end{pmatrix} 
        \quad 
        \text{and}
        \quad 
        \bm q
        =
        \begin{pmatrix}
            \bm q_{(1)} \\
            \bm q_{(2)} \\
            \vdots \\
            \bm q_{(L)}
        \end{pmatrix},
    \end{align*}
    where 
    \begin{equation}
        \bm p_{(l)} = {\rm vec} \left(F_{(l)} \circ (-{B}_{(l-1)} {W}_{(l)}  )\right),
        \label{p_l}
    \end{equation}
    \begin{equation}
        \bm q_{(l)} = {\rm vec} \left(F_{(l)}  \circ \left(  {A}_{(l-1)} {W}_{(l)} + \mathbf{1}_n \bm c_l^\top \right) \right).
        \label{q_l}
    \end{equation}    
Thus, the set $\cZ_u$  can be identified by solving a system of linear inequalities w.r.t $z$.
\end{proof}

\vspace{5pt}

The inequalities in Lemma \ref{lem:activation_function} can be derived iteratively by propagating the input ${X^s \choose X^t }$ through each layer of the feature extractor in the RL-based DA model and the AE, recording the linear constraints imposed by each activation function. This is possible because deep neural networks, typically composed of piecewise-linear activations such as ReLU, effectively accumulate linear transformations layer by layer. In the case of non-piecewise-linear activation functions, such as tanh or sigmoid, a similar discussion applies by considering their piecewise-linear approximations. 
    

\begin{lemma} \label{lem:ad_algorithm}
    The set $\cZ_v$, which represents sequence of steps used to identify anomalies based on the $\ell_1$ reconstruction error, can be characterized by a set of linear inequalities w.r.t. $z$:
    \begin{align*}
        \cZ_v 
        =
        \big \{
        z \in \mathbb{R}
        \mid
        \bm rz 
        \leq
        \bm t
        \big \},
    \end{align*}
    where vectors $\bm r$ and $\bm t$ are defined in Appendix \ref{appx:proof_ad_algorithm}.
\end{lemma}

\begin{lemma} \label{lem:test_statistic_calculation}
    The set $\cZ_{2}$ can be characterized by a set of linear inequalities w.r.t. $z$:
    \begin{align*}
        \cZ_{2} 
        =
        \big \{
        z \in \mathbb{R}
        \mid
        \bm wz 
        \leq
        \bm o
        \big \},
    \end{align*}
    where vectors $\bm w$ and $\bm o$ are defined in Appendix \ref{appx:proof_test_statistic_calculation}.
\end{lemma}

The proofs of Lemmas \ref{lem:ad_algorithm} and \ref{lem:test_statistic_calculation} are deferred to Appendices \ref{appx:proof_ad_algorithm} and \ref{appx:proof_test_statistic_calculation}, respectively. Due to the complexity and length of the definitions of $\bm{r}$, $\bm{t}$, $\bm{w}$, and $\bm{o}$, we also defer them to the appendices. In Lemmas \ref{lem:activation_function}, \ref{lem:ad_algorithm}, and \ref{lem:test_statistic_calculation}, we show that the truncation regions $\cZ_u$, $\cZ_v$, and $\cZ_2$ can each be characterized by a set of linear inequalities with respect to $z$. Once $\cZ_u$ and $\cZ_v$ are computed, the sub-problem region $\cZ_{u,v}$ in Eq.~\eqref{eq:extra_conditioning} can be obtained as the intersection $\cZ_{u,v} = \cZ_u \cap \cZ_v$.

\begin{algorithm}[!t]
\caption{\texttt{STAND-DA}}\label{stand-da-algorithm}
\begin{algorithmic}[1]
    \Require $X^s_{\rm obs}$, $X^t_{\rm obs}$, $z_{\rm min}$, $z_{\rm max}$
    \vspace{2pt}
    \State $\cO_{\rm obs}$ $\gets$ AE-based AD after RL-based DA on $\left (X^s_{\rm obs}, X^t_{\rm obs} \right)$
    \vspace{2pt}
    \For {$j \in \cO_{\rm obs}$}
    \vspace{2pt}
        \State Compute $\bm \eta_j \gets \text{Eq. } \eqref{eq:etaj}, {\bm a}$ and $\bm b \gets \text{Eq. } \eqref{eq:cZ}$
        \vspace{2pt}
        \State $\cZ \gets \texttt{divide\_and\_conquer}(\bm a, \bm b, z_{\rm min}, z_{\rm max}, \cO_{\rm obs})$
        \vspace{2pt}
        \State $p_j^{\rm selective} \gets$ Eq. \eqref{eq:valid_p_value_reformulated} with $\cZ$
        \vspace{2pt}
    \EndFor
    \vspace{2pt}
    \Ensure $\{p_j^{\rm selective}\}_{i \in \cO_{\rm obs}}$
\end{algorithmic}
\end{algorithm}

\begin{algorithm}[!t]
    \caption{\texttt{divide\_and\_conquer}}\label{line-search-algorithm}
    \begin{algorithmic}[1]
        \Require $\bm a$, $\bm b$, $z_{\rm min}$, $z_{\rm max}, \cO_{\rm obs}$
        \vspace{2pt}
        \State \textbf{Initialization:} $z = z_{\rm min}, \cZ_1 = \emptyset$
        \vspace{2pt}
        \While {$z \leq z_{\rm max}$}
        \vspace{2pt}
            \State $\cO_{\bm a + \bm b z} \gets$ AE-based AD after RL-based DA on $\bm a + \bm bz$
            \vspace{2pt}
            \State Compute $\cZ_u \gets$ Lemma \ref{lem:activation_function} and $\cZ_v \gets$ Lemma \ref{lem:ad_algorithm}
            \vspace{2pt}
            \State $\cZ_{u, v} = [\ell^z, r^z] \gets \cZ_u \cap \cZ_v$
            \vspace{2pt}
            \If{$\cO_{\bm a + \bm b z} = \cO_{\rm obs}$} 
            \vspace{2pt}
                \State $\cZ_1 \gets \cZ_1 \cup \cZ_{u, v}$
            \EndIf 
            \vspace{2pt}
            \State $z \gets r^z + \Delta z$ $\quad $// Step past $r^z$ by adding a small $\Delta z$ (e.g., 0.001)
            \vspace{2pt}
        \EndWhile
        \vspace{2pt}
        \State $\cZ_2 \gets$ Lemma \ref{lem:test_statistic_calculation}
        \vspace{2pt}
        \State $\cZ = \cZ_1 \cap \cZ_2$
        \vspace{2pt}
        \Ensure $\cZ$
    \end{algorithmic}
\end{algorithm}

\vspace{8pt}
\textbf{Computing $\cZ$ by integrating solutions from multiple sub-problems.}
To identify $\cZ$ in \eqref{eq:cZ}, the RL-based DA and the AE-based AD after DA are repeatedly applied to a sequence of datasets of the form $\bm a + \bm b z$, across a sufficiently wide range of values $z \in [z_{\rm min}, z_{\rm max}]$.
This divide-and-conquer procedure is outlined in Algorithm~\ref{line-search-algorithm}.
Once $\cZ$ is obtained via Algorithm~\ref{line-search-algorithm}, the proposed selective $p$-value in \eqref{eq:valid_p_value_reformulated} can be computed.
The complete procedure of the proposed STAND-DA method is summarized in Algorithm~\ref{stand-da-algorithm}.

\section{Implementation: GPU-Accelerated STAND-DA} \label{sec:extension} 

In the context of SI for DNN-related problems, computational efficiency is a critical concern due to the intensive nature of the underlying procedures. 
A major bottleneck arises from the repeated forward passes through the network during the line search process (Algorithm~\ref{line-search-algorithm}), where each step may involve evaluating many data points. These repeated processes can quickly become prohibitively expensive, particularly for large-scale models or datasets.
To address this, GPU acceleration is employed to substantially reduce the computational burden. By exploiting the highly parallelizable structure of matrix operations inherent in DNNs, GPUs can significantly speed up forward propagation---making SI tractable even for deep and wide architectures.
This section details how GPU-based acceleration is integrated into the proposed  STAND-DA method.

\subsection{Framework and Setup}


We accelerated our method using Numba’s CUDA extension ({\tt numba.cuda}), which provides a low-level API for designing custom GPU kernels directly in Python. Unlike high-level frameworks such as PyTorch and TensorFlow, which encapsulate matrix operations within abstract, general-purpose APIs, {\tt numba.cuda} offers fine-grained control over data transfer between host (CPU) and device (GPU) as well as customizable parallelism. This control allows us to tailor both computation and memory layout to the specific structure of our problem, yielding significantly improved performance and reduced overhead.

\vspace{5pt}


To combine the flexibility of PyTorch’s model definition with the performance advantages of Numba-CUDA kernels, we extract trained parameters from a PyTorch-defined neural network and convert them into Numba-CUDA device arrays. The conversion process involves parsing the model’s layer structure, extracting each layer’s weights and biases, and storing them as \{layer name, device array\} pairs. By completing this step before the line-search phase, we can directly use these parameters in our custom kernels without incurring additional data transfer costs, tailoring them for GPU-based selective inference using {\tt numbda.cuda}.

\begin{algorithm}[!t]
    \caption{\texttt{parse\_network}}\label{alg:convert-network}
    \begin{algorithmic}[1]
        \Require Trained PyTorch model $\mathcal{N}$ 
        \vspace{2pt}
        \State \textbf{Initialize:} $\texttt{list\_layers}=\emptyset$
        \vspace{2pt}
        \ForAll{(\texttt{name}, \texttt{param}) in $\mathcal{N}$}
            \vspace{2pt}
            \If{\texttt{name} = \texttt{`Linear Weight'}}
                \vspace{2pt}
                \State \texttt{weight} $\gets$ Numba-CUDA's device array of \texttt{param}
                \vspace{2pt}
                \State Append(\{\texttt{`Linear Weight'}, \texttt{weight}\}) to \texttt{list\_layers}
                \vspace{2pt}
            \ElsIf{\texttt{name} = \texttt{`Linear Bias'}}
                \vspace{2pt}
                \State \texttt{bias} $\gets$ Numba-CUDA's device array of \texttt{param}
                \vspace{2pt}
                \State Append(\{\texttt{`Linear Bias'}, \texttt{bias}\}) to \texttt{list\_layers}
                \vspace{2pt}
            \ElsIf{\texttt{name} = \texttt{`ReLU'}}
                \vspace{2pt}
                \State Append(\{\texttt{`ReLU'}, \texttt{None}\}) to \texttt{list\_layers}
                \vspace{2pt}
            \EndIf
            \vspace{2pt}
        \EndFor
        \vspace{2pt}
        \Ensure Parsed Network \texttt{list\_layers}
    \end{algorithmic}
\end{algorithm}

\vspace{5pt}

\vspace{5pt}

Algorithm \ref{alg:convert-network} outlines the conversion process: the model’s layers are parsed, weights and biases from Linear layers are extracted, and all parameters are transferred to GPU memory. In this work, the network is assumed to follow a simple feed-forward architecture alternating between Linear and ReLU layers. Although the current implementation supports only Linear and ReLU layers in PyTorch-defined models, it can be easily extended to include additional activation or normalization layers, as well as models from other frameworks such as Keras and TensorFlow.

\subsection{CUDA Acceleration of the Proposed STAND-DA Method}

To accelerate the repeated matrix computations and interval updates during the feed-forwarding process of the DNN, we implemented all key operations as custom GPU kernels using {\tt numba.cuda}. This custom design offloads the computational workload to the GPU and reduces the time overhead from frequent data transfers between host and device, ensuring that intermediate results remain in GPU memory throughout the computation.

\vspace{5pt}

The core acceleration targets the following steps:
\begin{itemize}
    \item \textbf{Affine transformations:} We designed custom matrix affine transformation kernels to accelerate forward passes through the network during the line search in Algorithm \ref{line-search-algorithm}. Each transformation is decomposed into a matrix multiplication followed by a bias addition, both implemented using custom CUDA kernels:
   
    \begin{itemize}
        \item \texttt{MatMulMat}: a shared-memory–tiled CUDA kernel for efficient matrix multiplication.
        \item \texttt{MatAddBias}: a 1D-grid kernel that adds biases column-wise across matrix rows.
    \end{itemize}
    This customization enables the linear operations in \eqref{eq:X_l}, \eqref{A_l} and \eqref{B_l} to execute in parallel on the GPU. As a result, we achieved significant speedups over high-level frameworks by explicitly controlling data transfers between the host (CPU) and the device (GPU).
    
\vspace{5pt}

    \item \textbf{Conditional via {\tt siReLU}:} We introduce a specialized kernel, \texttt{siReLU} (SI for ReLU), to perform ReLU on the data while conditioning on the activation pattern of ReLU during forward propagation. This design allows us to feed-forward the data and dynamically condition on the event in Lemma \ref{lem:activation_function}, thereby accelerating the computation of the corresponding interval. Specifically, this kernel:
    \begin{itemize}
        \item Applies ReLU to $X^+_{(l)}$, where $X^+_{(l)}=X_{(l-1)} W_{(l)} + \mathbf{1}_n\,\mathbf{c}_{(l)}^\top$ denotes the linear transformation of $X_{(l-1)}$ after the $l$-th linear layer.
        \item Forwards the linear representation matrices of $X^+_{(l)}$ through the $l$-th ReLU layer.
        \item Dynamically updates the interval bounds for conditional event in Lemma \ref{lem:activation_function} while performing feed-forwarding through the layers. This interval is updated by solving the inequality in \eqref{F_l_ineq} for each element:
        \[
        \left[F_{(l)}\right]_{ij} \circ \left(\left[A^+_{(l)}\right]_{ij} + \left[B^+_{(l)}\right]_{ij}z\right) \geq 0,
        \]
        where $A^+_{(l)}=A_{(l-1)}W_{(l)}+\bm 1_n \bm c^\top_{(l)}$ and $B^+_{(l)}=B_{(l-1)}W_{(l)}$ represent linear transformation matrices of $X^+_{(l)}$ w.r.t scalar $z$, i.e., $X^+_{(l)}=A^+_{(l)}+B^+_{(l)}z$.
    \end{itemize}
\end{itemize}

Algorithm \ref{siReLU_algorithm} details the implementation of {\tt siReLU} kernel. This kernel forwards both $X^+_{(l)}$ and its parameterized form through the $l$-th ReLU layer. In this way, it directly integrates the conditioning on each layer’s activation patterns into the data feed-forward process. This specialized design enables the interval bounds associated with these conditions to be dynamically updated layer by layer, avoiding separate post-processing steps and ensuring tight bounds throughout the forward propagation. Furthermore, by leveraging Numba-CUDA’s parallel execution model, the for-loop over each $(i, j)$ matrix element is mapped to GPU threads, allowing simultaneous processing of all elements and significantly reducing execution time compared to sequential CPU iteration. While the current implementation is customized for ReLU, the design of this {\tt siReLU} kernel can be easily extended to other piecewise-linear or even non-linear activation functions, as discussed in \S\ref{subsec:truncation_region}.

\begin{algorithm}[!t]
    \caption{\texttt{siReLU}}\label{siReLU_algorithm}
    \begin{algorithmic}[1]
        \Require $X^+_{(l)}$, $A^+_{(l)}$, $B^+_{(l)}$, $\left[l, \, r\right]$
        \vspace{2pt} 
        \State \textbf{Initialization:} $X_{(l)} \gets X^+_{(l)}$, $A_{(l)} \gets A^+_{(l)}$, $B_{(l)} \gets B^+_{(l)}$
            \vspace{2pt}
            \State $\mathcal{F}_{(l)} \gets$ Eq. \eqref{F_l}
            \vspace{2pt}
            \For{$i \in [n_s+n_t]$, $j \in [d_l]$} $\quad$//Executed in parallel across $(i, j)$
                \vspace{2pt}
                    \If{$\left[{\mathcal{F}}_{(l)}\right]_{ij}==-1$}
                    \vspace{2pt}
                    \State $\left[X_{(l)}\right]_{ij} \gets 0$, $\left[A_{(l)}\right]_{ij} \gets 0$, $\left[B_{(l)}\right]_{ij} \gets 0$
                    \vspace{2pt}
                    \EndIf
                \vspace{2pt}
                \If{$\left[{\mathcal{F}}_{(l)}\right]_{ij} * \left[B^+_{(l)}\right]_{ij} > 0$}
                \vspace{2pt}
                \State $l \gets {\rm max}\Big(l, -\left[A^+_{(l)}\right]_{ij} / \left[B^+_{(l)}\right]_{ij}\Big)$
                \vspace{2pt}
                \Else
                \vspace{2pt}
                \State $r \gets {\rm min}\Big(r, -\left[A^+_{(l)}\right]_{ij} / \left[B^+_{(l)}\right]_{ij}\Big)$
                \vspace{2pt}
                \EndIf 
                \vspace{2pt}    
            \EndFor
        \vspace{2pt}
        \Ensure $X_{(l)}$, $A_{(l)}$, $B_{(l)}$, $\left[l,\, r\right]$
    \end{algorithmic}
\end{algorithm}

\vspace{5pt}

\begin{figure*}[!t]
\centering
\includegraphics[width=.85\textwidth]{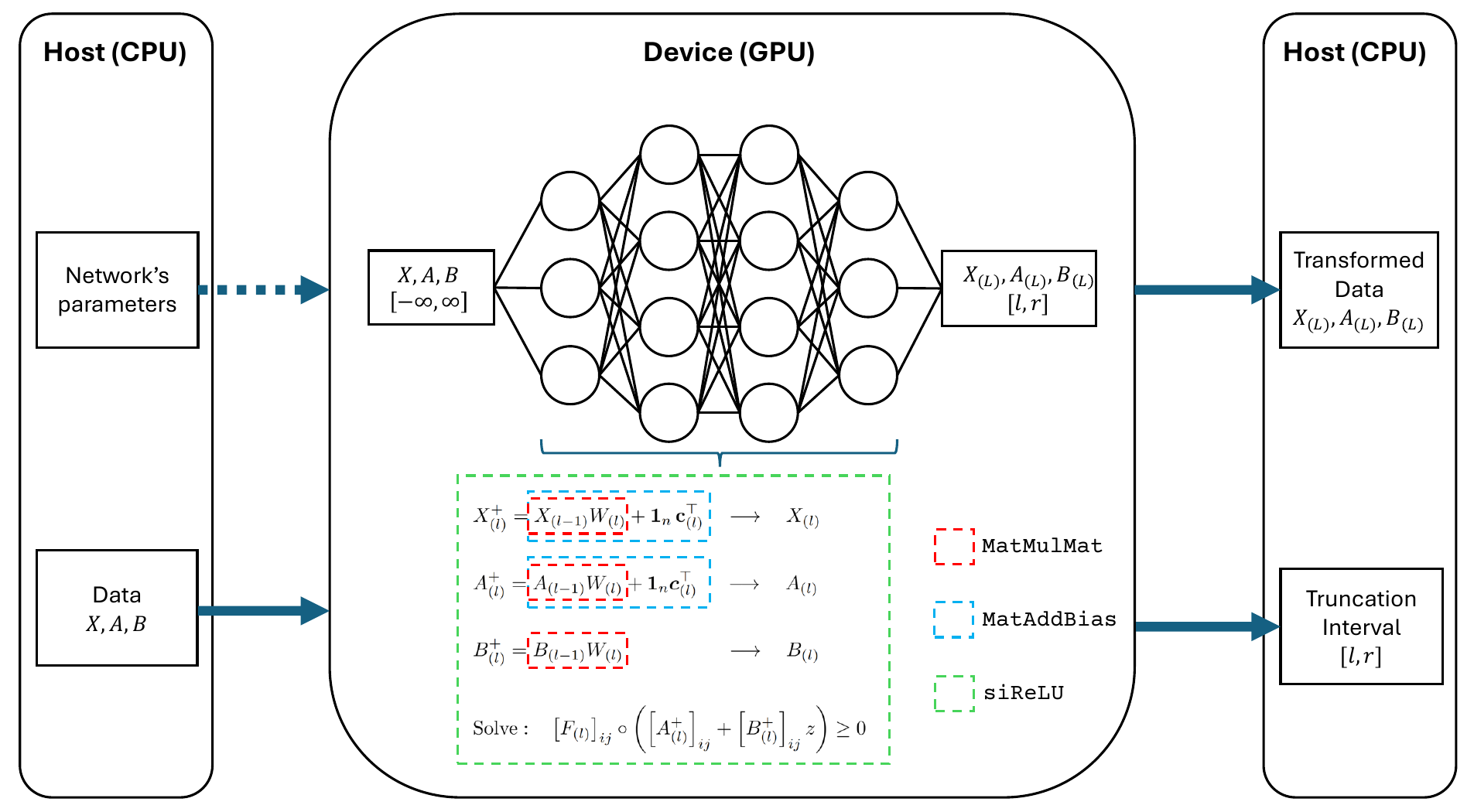}

\vspace{20pt}
\caption{Illustration of GPU acceleration in our proposed method. The network parameters, along with the input data and its linear representations, are transferred to the GPU. The feed-forward computations are executed using our custom kernels: {\tt MatMulMat}, {\tt MatAddBias}, {\tt siReLU}. During propagation through each layer, the input data, its parameterized form, and the interval bounds are dynamically updated.}
    \label{fig:gpu-accelerate}
\vspace{-10pt}
\end{figure*}

The entire feed-forward process is illustrated in Fig. \ref{fig:gpu-accelerate}. These GPU kernels are orchestrated at the layer level using separate CUDA streams to exploit asynchronous execution. For example, when multiplying by a weight matrix, the operations for $X$, $A$, and $B$ can run in parallel on different streams without waiting for one another, thereby improving overall throughput. Moreover, all intermediate arrays ($X$, $A$, $B$) are pre-allocated and reused, reducing memory allocation overhead and minimizing costly host–device data transfers throughout the inference phase. This low-level, fused-kernel approach enables our implementation to handle deep architectures efficiently during the repeated evaluations required by the line search, yielding substantial speedups over naïve CPU-based or high-level frameworks-based alternatives. The implementation of these kernels is available in the accompanying package at \url{https://github.com/DAIR-Group/STAND-DA}.


\section{Experiments}\label{sec5}
In this section, we demonstrate the performance of our proposed STAND-DA method. We compared the performance of the following methods in terms of FPR and TPR:
\begin{itemize}
    \item $\texttt{STAND-DA}$: proposed method 
    
    \item $\texttt{STAND-DA-oc}$: proposed method with only the extra-conditioning mentioned described in \S \ref{subsec:truncation_region}
    
    \item $\texttt{Naive}$: traditional statistical inference
    
    \item $\texttt{Bonferroni}$: the most popular multiple hypothesis testing approach
    
    \item $\texttt{No Inference}$: AD after DA without inference
\end{itemize}

We note that any method failing to control the FPR below a specified threshold $\alpha$ is deemed statistically invalid, rendering its TPR irrelevant for evaluation. A high TPR simply implies a low FNR. In all experiments, we set the significance level at $\alpha = 0.05$.
For model architecture, both the feature extractor of RL-based DA and the AE-based AD follow a multi-layer perceptron (MLP) setup, consisting of fully connected layers with ReLU activations---an approach commonly employed for its practicality and strong empirical performance. The feature extractor is structured with two hidden layers containing 500 and 100 neurons. The domain critic is implemented as a single-layer MLP with 100 hidden units.
The AE is built symmetrically: the encoder comprises seven layers with decreasing dimensions $\big [100, 64, 32, 16, 8, 4, 2 \big ]$, and the decoder mirrors this layout in reverse as $\big [2, 4, 8, 16, 32, 64, 100 \big ]$. All hidden layers apply ReLU activations, except for the decoder's output layer. To quantify reconstruction quality, we minimize the $\ell_1$ reconstruction loss. Anomalies are identified by ranking samples according to their reconstruction error and labeling the top 5\% as outliers, using the $95^{\rm th}$ percentile as the anomaly threshold.

\subsection{Numerical Experiments}\label{subsec:numerical_experiments}

\textbf{Independent data.} 
We synthesized the source and target datasets $X^s$ and $X^t$ as follows. Each source sample $X^s_{i,:}$ was drawn from a $d$-dimensional standard normal distribution, i.e., $X^s_{i,:} \sim \mathbb{N}(\bm{0}_d, I_d)$ for all $i \in [n_s]$. Similarly, target samples $X^t_{j,:}$ were generated from a normal distribution with shifted mean, $X^t_{j,:} \sim \mathbb{N}(\bm{2}_d, I_d)$ for all $j \in [n_t]$. We fixed the feature dimension at $d = 10$. The training dataset consisted of $n_s = 1000$ source samples and $n_t = 100$ target samples. To introduce anomalies, $5\%$ of the instances in each domain were randomly chosen and perturbed by adding a constant offset $\Delta$ to the corresponding mean vector components.
For the evaluation of FPR, we varied the source sample size as $n_s \in \{50, 100, 150, 200\}$, fixed $n_t = 25$, and set  $\Delta = 0$. The TPR analysis was conducted with $n_s = 150$, $n_t = 50$, and varying $\Delta \in \{0.5, 1.0, 1.5, 2.0\}$. 

\vspace{5pt}

The results are visualized in Fig.~\ref{fig:fpr_tpr_10d}.
Each marker in the plots denotes the empirical probability $\mathbb{P}(\text{$p$-value} \leq \alpha)$, estimated by the proportion of 120 simulation runs in which the computed $p$-value falls below the threshold $\alpha$.
In the left subplot, we observe that \texttt{STAND-DA}, \texttt{STAND-DA-oc}, and the \texttt{Bonferroni} correction successfully maintain FPR control under the significance level $\alpha$. However, the \texttt{Naive} and \texttt{No Inference} baselines fail to meet this criterion and thus were excluded from further consideration in the TPR analysis.
Turning to the right subplot, we find that \texttt{STAND-DA} consistently achieves the highest TPR across all levels of $\Delta$, indicating superior statistical power. This corresponds to the lowest FNR among all evaluated methods.

\begin{figure}[!t]
     \centering
     \begin{subfigure}[b]{0.492\linewidth}
         \centering
         \includegraphics[width=\textwidth]{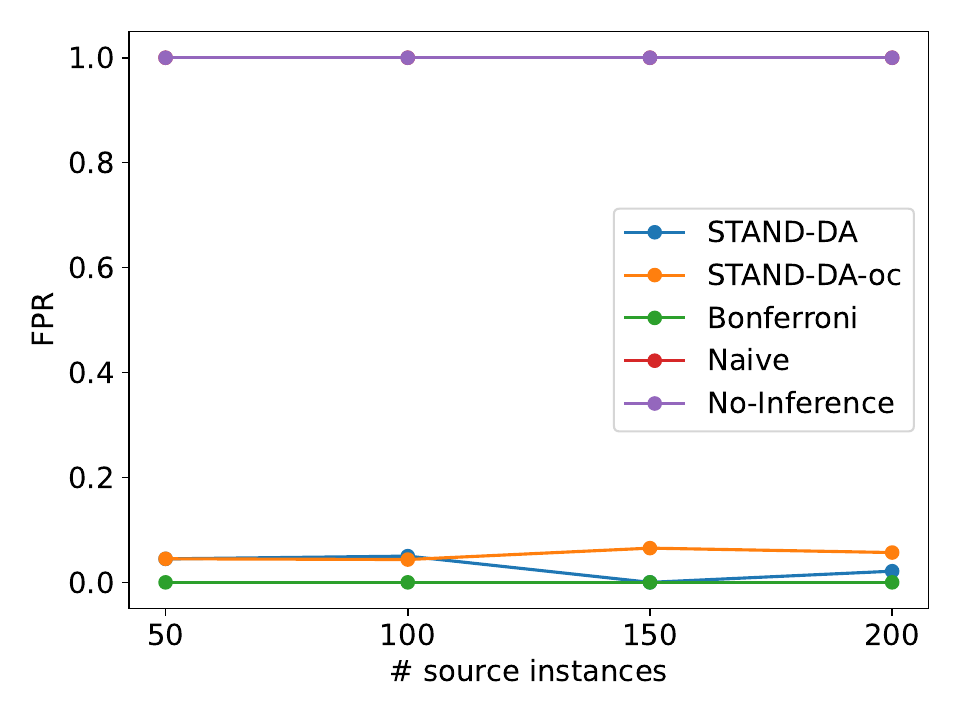}
         \caption{FPR}
     \end{subfigure}
     \hfill
     \begin{subfigure}[b]{0.492\linewidth}
         \centering
         \includegraphics[width=\textwidth]{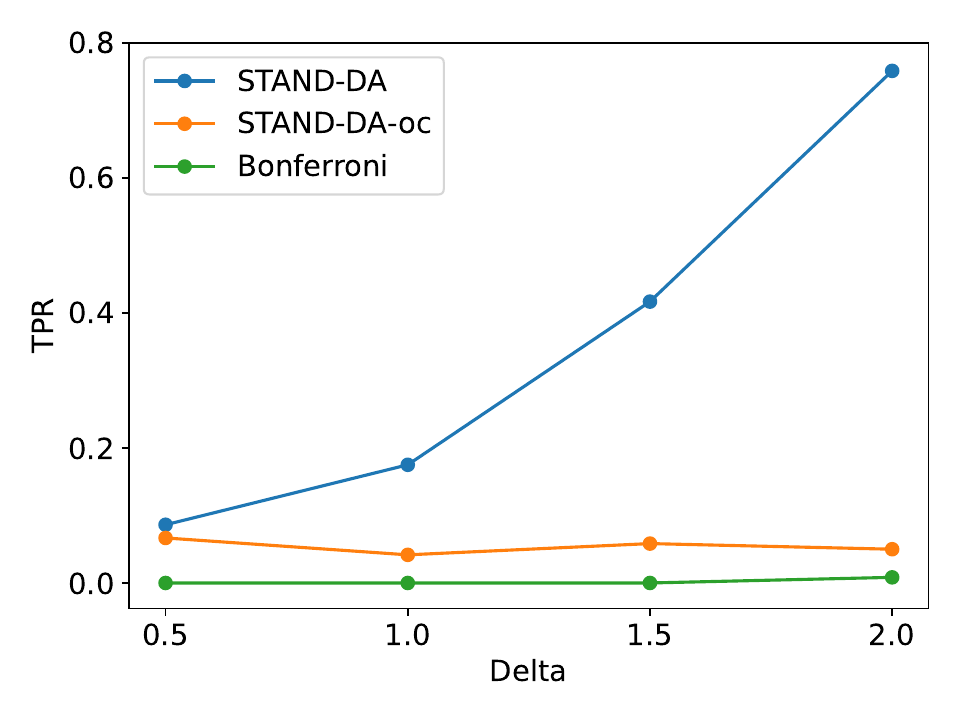}
         \caption{TPR}
     \end{subfigure}
     \caption{FPR and TPR in the case of independent data}
     \label{fig:fpr_tpr_10d}
     \vspace{-8pt}
\end{figure}

\vspace{5pt}

\begin{figure}[!t]
     \centering
     \begin{subfigure}[b]{0.492\linewidth}
         \centering
         \includegraphics[width=\textwidth]{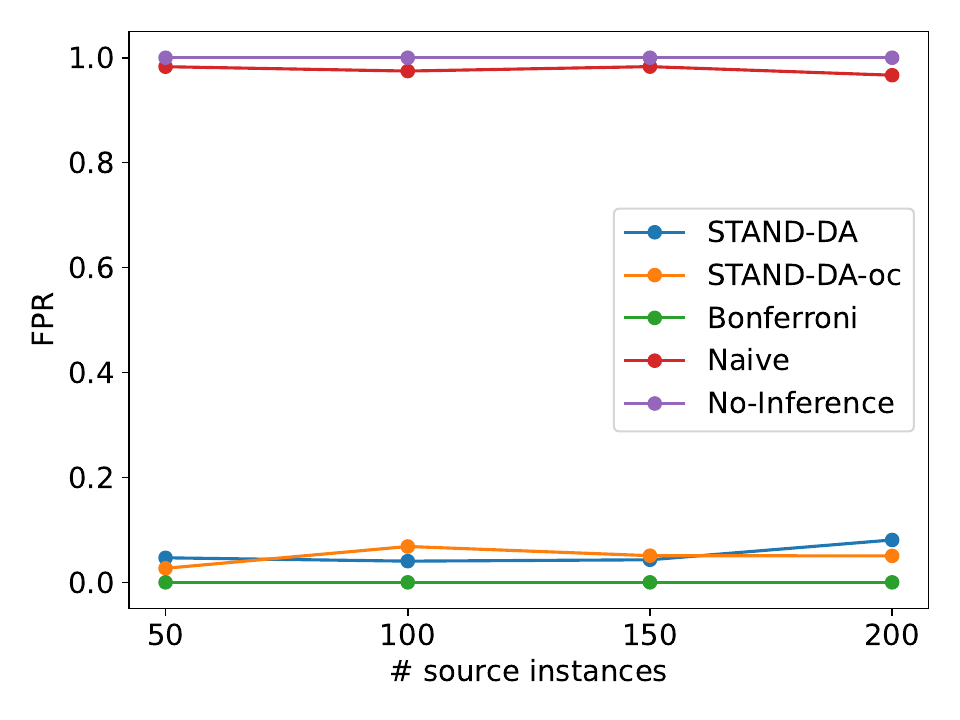}
         \caption{FPR}
     \end{subfigure}
     \hfill
     \begin{subfigure}[b]{0.492\linewidth}
         \centering
         \includegraphics[width=\textwidth]{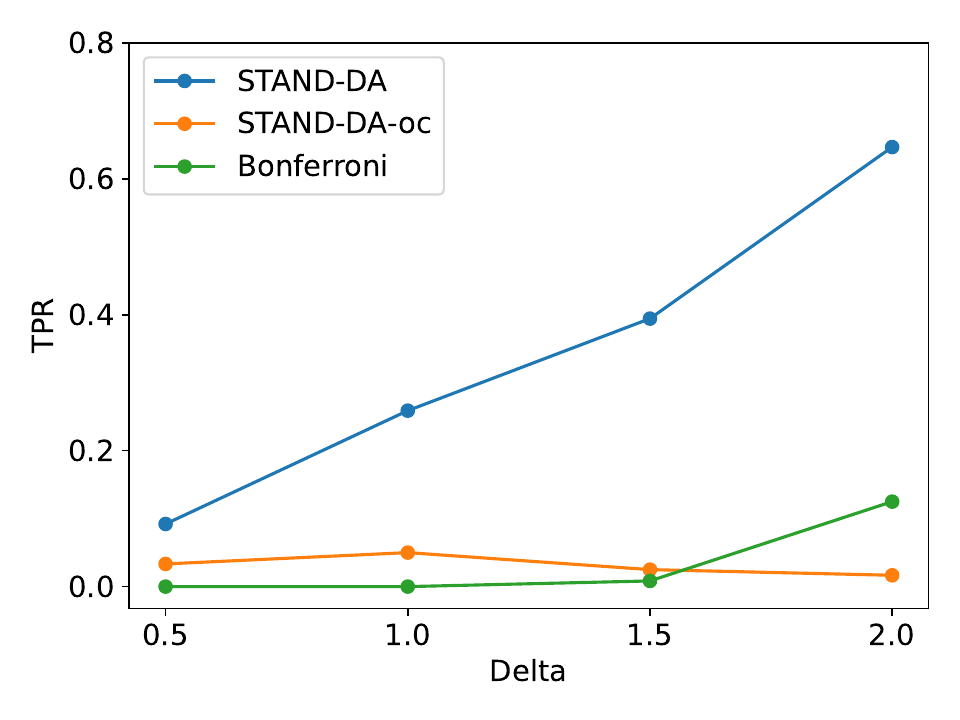}
         \caption{TPR}
     \end{subfigure}
     \caption{FPR and TPR in the case of correlated data}
     \label{fig:fpr_tpr_correlated}
\end{figure}

\textbf{Correlated data.} 
In this scenario, we explore a setting where the data exhibits internal correlation. Specifically, we construct the source and target datasets such that each row of the source data matrix \( X^s \) follows a multivariate normal distribution \( X^s_{i, :} \sim \mathbb{N}(\mathbf{0}_d, \Xi) \) for all \( i \in [n_s] \), while each row of the target data matrix \( X^t \) is sampled from \( \mathbb{N}(\mathbf{2}_d, \Xi) \) for all \( j \in [n_t] \). The covariance matrix \( \Xi \in \mathbb{R}^{d \times d} \) is defined as \( \Xi_{ij} = \rho^{|i - j|} \), with \( \rho = 0.5 \) and dimensionality \( d = 10 \). All other experimental parameters for measuring FPR and TPR mirror those used in the independent data setting. The results, shown in Fig.~\ref{fig:fpr_tpr_correlated}, demonstrate that our proposed \texttt{STAND-DA} method consistently achieves the highest TPR while controlling the FPR across all trials.

\begin{figure}[!t]
    \centering 
    \centering
     \begin{subfigure}[b]{0.492\linewidth}
         \centering
         \includegraphics[width=\textwidth]{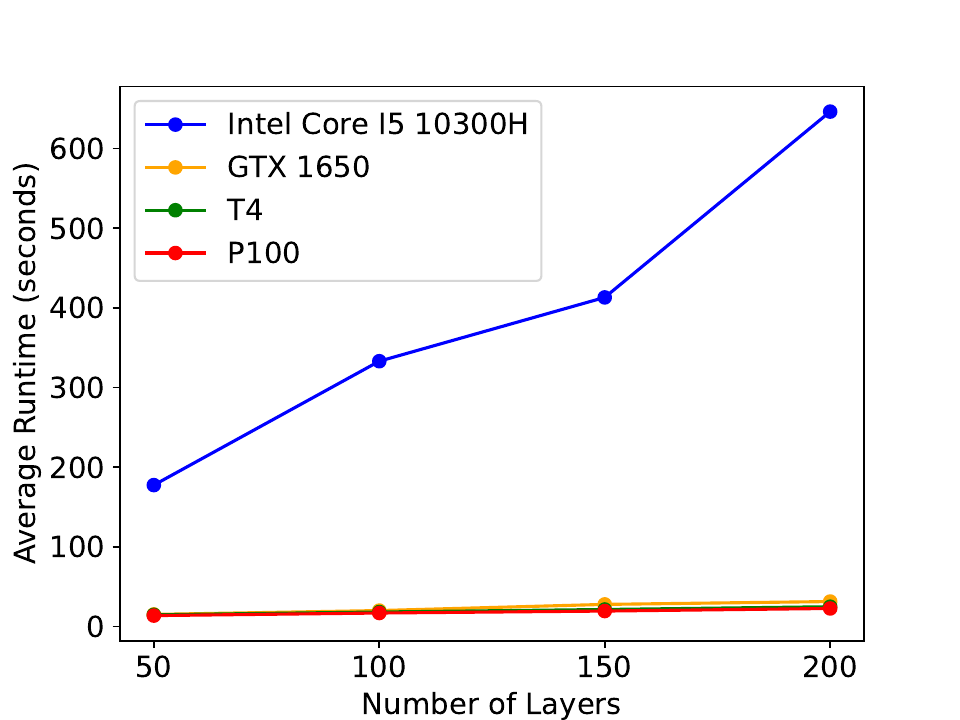}
         \caption{Execution time on all devices}
     \end{subfigure}
     \hfill
     \begin{subfigure}[b]{0.492\linewidth}
         \centering
         \includegraphics[width=\textwidth]{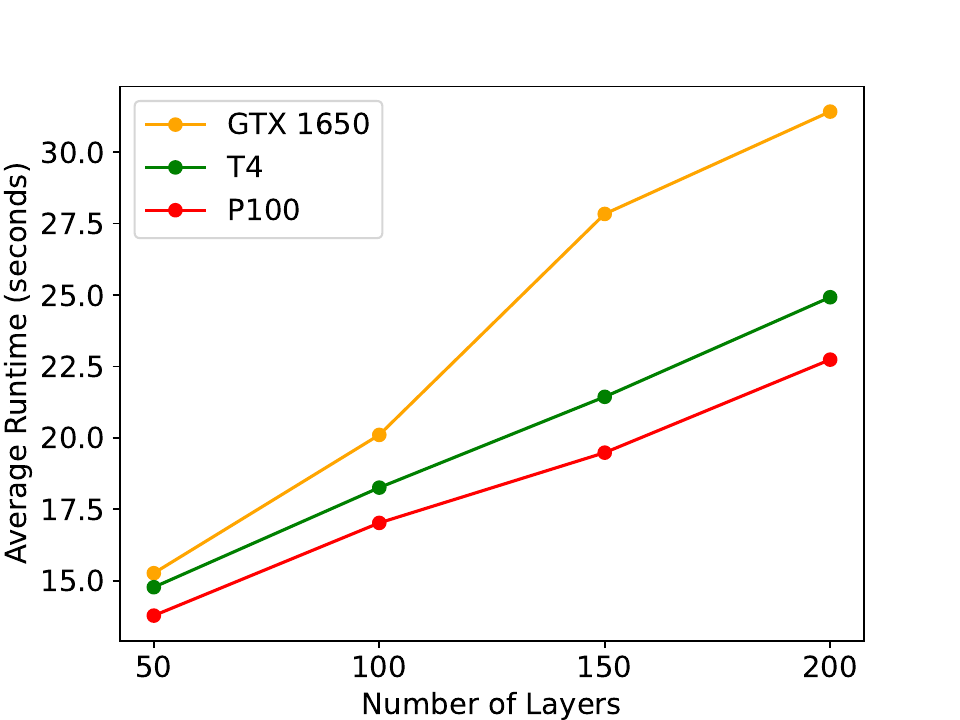}
         \caption{Execution time on only GPUs}
     \end{subfigure}
    \caption{Execution time on different devices}
    \label{fig:execution_time_devices}
\end{figure}

\vspace{8pt}
\textbf{Computational time.} 
To evaluate the efficiency of our proposed {\tt STAND-DA} framework, we measured its execution time across various computing environments, including a standard Intel Core i5-10300H CPU and three different GPUs: NVIDIA GTX 1650, T4, and P100. The feature extractor architecture was designed with structure of $[500, 100, 128]$. The AE architecture was designed with an encoder structure of $[128] \times l/2$ followed by $[64, 32, 16, 8, 4, 2]$, and a decoder mirroring this pattern in reverse, ending with $[128] \times l/2$, where $l$ took values in $\{50, 100, 150, 200\}$ and the sample size was fixed at $n_s = 150$, the data dimension was also fixed at $d=128$. All other settings aligned with those used in our experiments with independently drawn data. As illustrated in Fig.~\ref{fig:execution_time_devices}, we observed a linear growth in execution time as the number of layers increased. Importantly, GPU-based executions consistently outperformed the CPU, with the T4 and P100 showing faster runtimes than the GTX 1650, especially on the higher number of layers. This highlights the clear advantage of deploying STAND-DA on high-performance, server-grade GPUs, which are specifically engineered for demanding machine learning applications. The results affirm that our method scales efficiently when leveraging GPU architectures.

\begin{figure}[!t]
    \centering 
    \centering
     \begin{subfigure}[b]{0.492\linewidth}
         \centering
         \includegraphics[width=\textwidth]{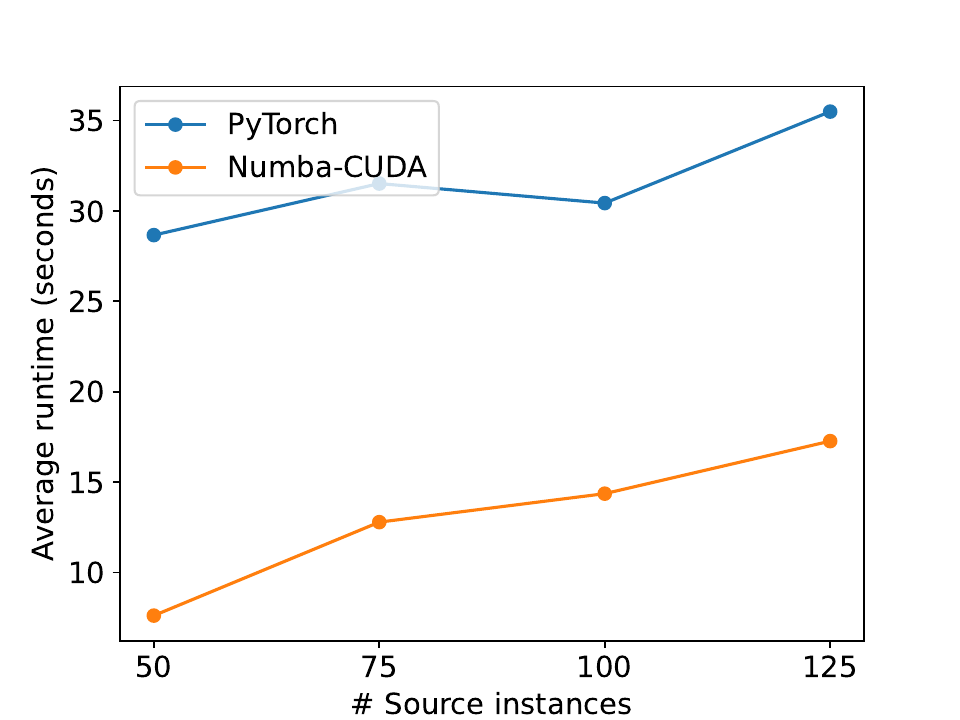}
         \caption{Execution time vs. \# Source instances}
         \label{fig:compare_pytorch_ns}
     \end{subfigure}
     \hfill
     \begin{subfigure}[b]{0.492\linewidth}
         \centering
         \includegraphics[width=\textwidth]{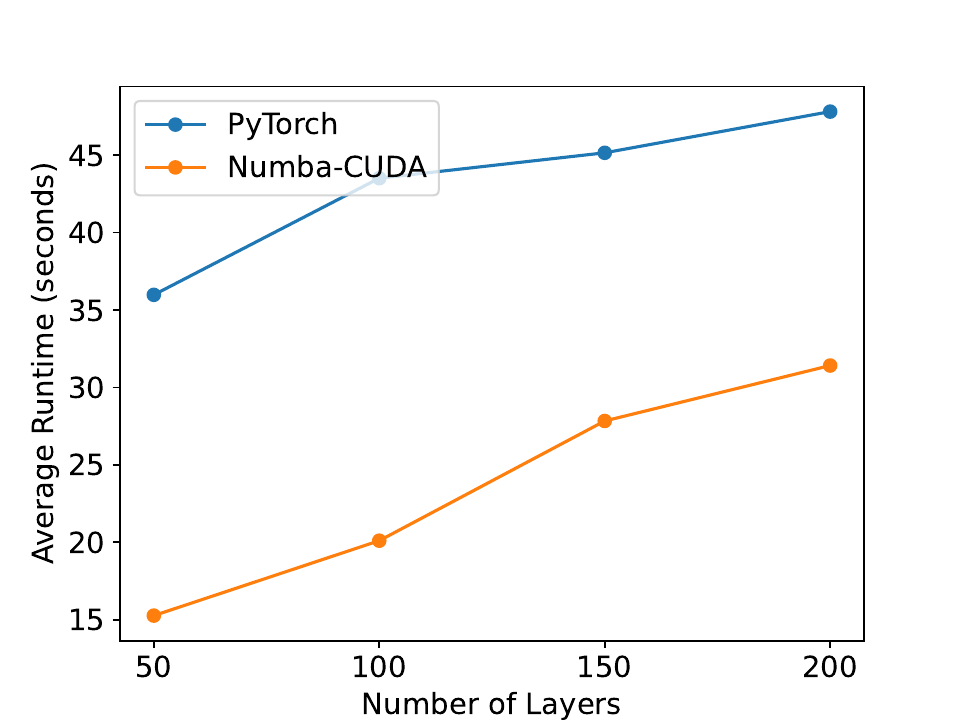}
         \caption{Execution time vs. Number of layers}
         \label{fig:compare_pytorch_layers}
     \end{subfigure}
    \caption{Execution time on different implementations}
    \label{fig:compare_pytorch}
\end{figure}

\vspace{5pt}

\textbf{Efficiency of GPU acceleration.} We evaluate the efficiency of our custom GPU kernels by comparing the total time required to compute a $p$-value against a PyTorch-based implementation. The results are presented in Fig. \ref{fig:compare_pytorch}. In Fig. \ref{fig:compare_pytorch_ns}, we set the number of source samples $n_s \in \{50, 75, 100, 125\}$ is varied to represent increasing dataset sizes, while other configurations follow those used in the FPR experiments. In Fig. \ref{fig:compare_pytorch_layers}, we set the number of layers $l\in\{50, 100, 150, 200\}$ to examine scalability with respect to network width and depth, with other configurations identical to those in the computational time experiments. Across both settings, the PyTorch implementation exhibits substantially higher computational time than our implementation using Numba-CUDA kernels, highlighting the superior efficiency of our specialized kernels for this problem.

\subsection{Real-data Experiments}\label{subsec:real_world_experiments}

\begin{figure}[!t]
     \centering
     \begin{subfigure}[b]{0.492\linewidth}
         \centering
         \includegraphics[width=\textwidth]{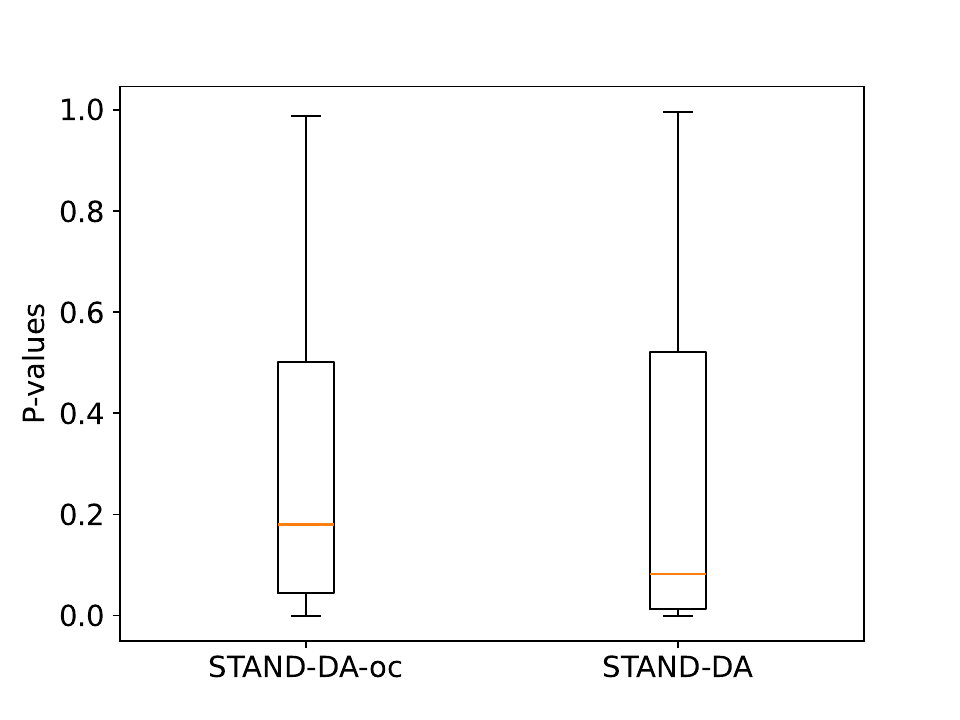}
         \caption{Heart Disease Dataset}
     \end{subfigure}
     \hfill
     \begin{subfigure}[b]{0.492\linewidth}
         \centering
         \includegraphics[width=\textwidth]{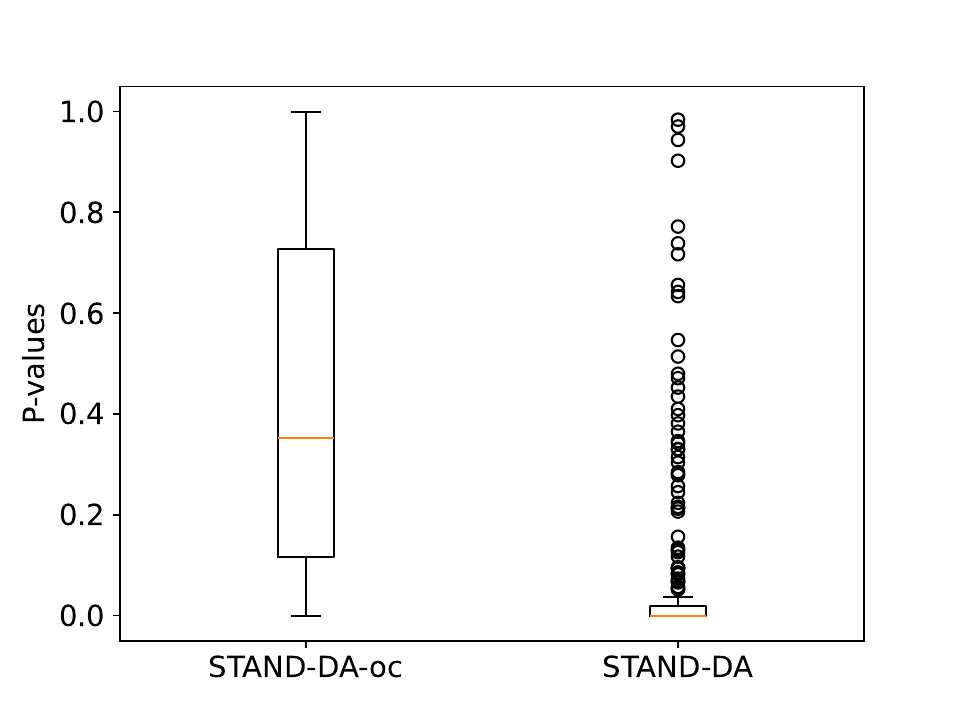}
         \caption{Breast Cancer Wisconsin Dataset}
     \end{subfigure}
     \hfill
     \begin{subfigure}[b]{0.492\linewidth}
         \centering
         \includegraphics[width=\textwidth]{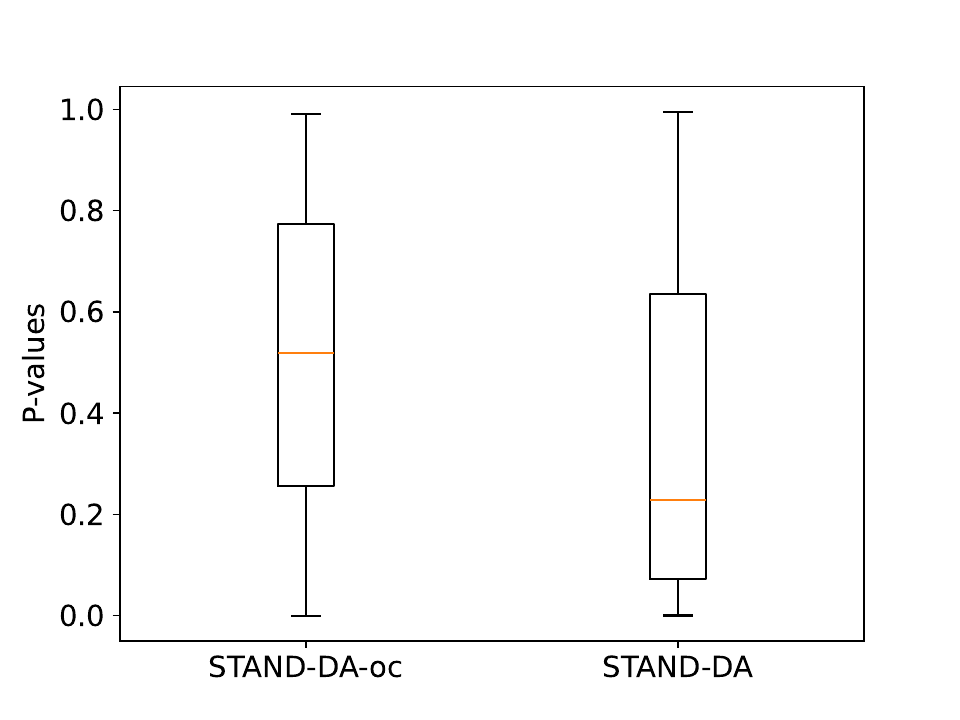}
         \caption{CDC Diabetes Health Indicators Dataset}
     \end{subfigure}
     \caption{Boxplots of $p$-values on real datasets}
     \label{fig:real_exp_boxplots}
\end{figure}

\begin{figure}[!t]
     \centering
     \begin{subfigure}[b]{0.492\linewidth}
         \centering
         \includegraphics[width=\textwidth]{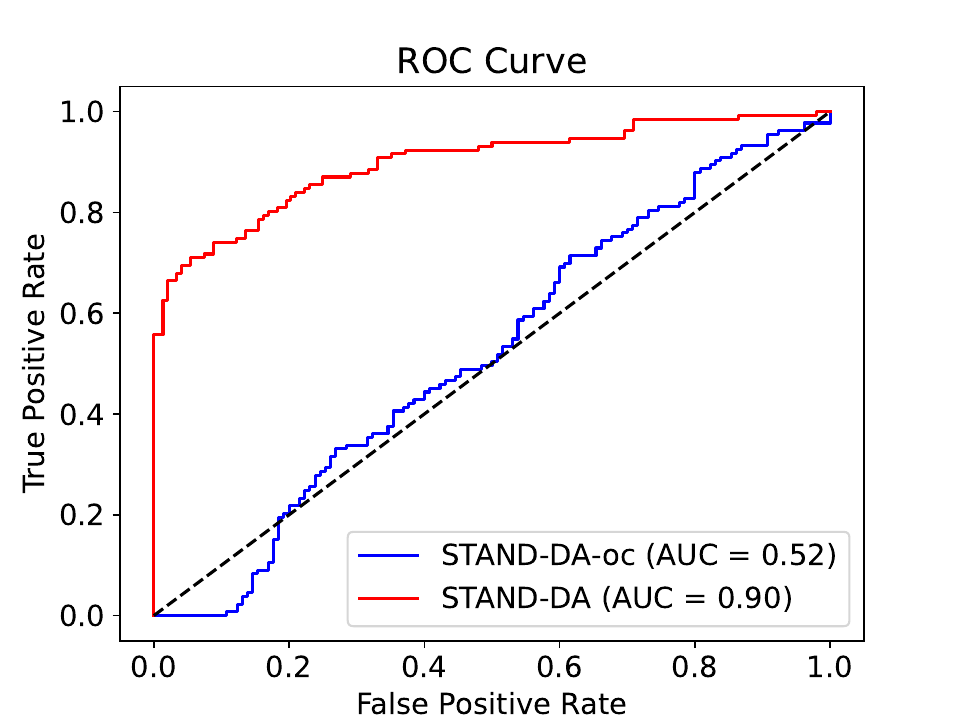}
         \caption{Heart Disease Dataset}
     \end{subfigure}
     \hfill
     \begin{subfigure}[b]{0.492\linewidth}
         \centering
         \includegraphics[width=\textwidth]{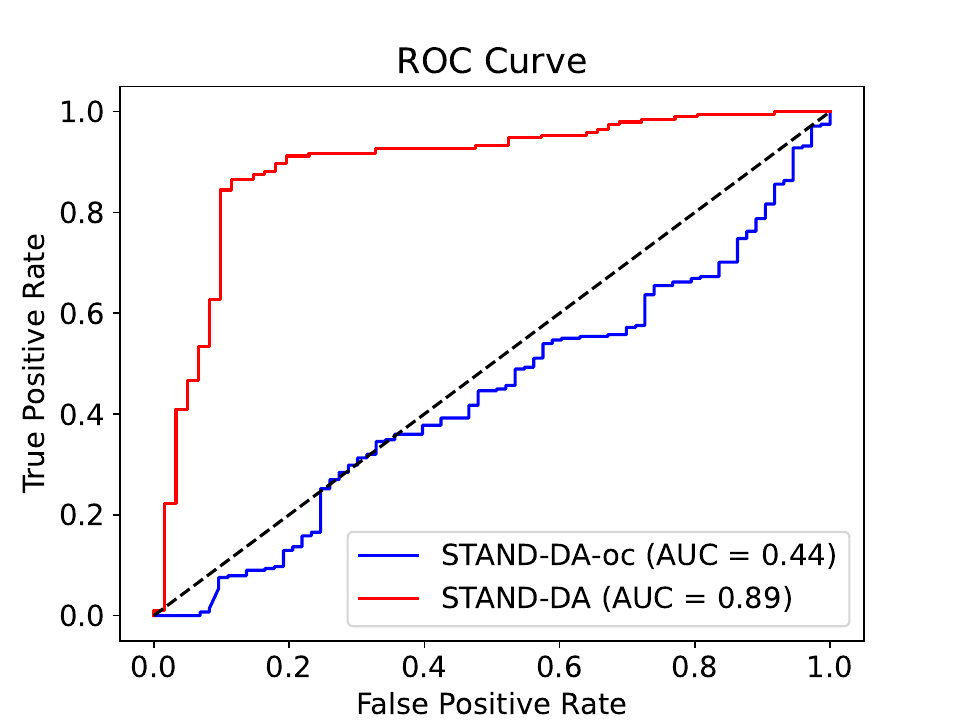}
         \caption{Breast Cancer Wisconsin Dataset}
     \end{subfigure}
     \hfill
     \begin{subfigure}[b]{0.492\linewidth}
         \centering
         \includegraphics[width=\textwidth]{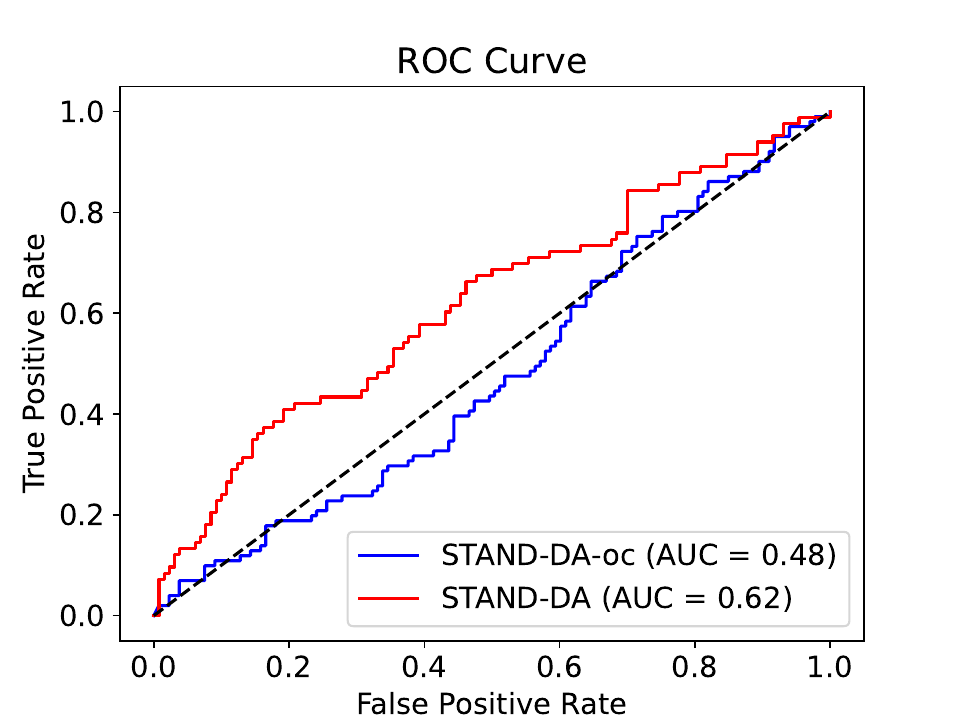}
         \caption{CDC Diabetes Health Indicators Dataset}
     \end{subfigure}
     \caption{ROC curves on real datasets}
     \label{fig:real_exp_roc_curves}
\end{figure}

We conducted comparisons on three real-world datasets: Heart Disease, Breast Cancer Wisconsin (Original) and
CDC Diabetes Health Indicators all available at the UCI Machine Learning Repository. We compared the $p$-values and ROC-curves of the $\texttt{STAND-DA}$ and $\texttt{STAND-DA-oc}$. 
\begin{itemize}
    \item \textbf{Heart Disease Dataset} contains 303 samples with 13 features, the goal is to predict presence of heart disease. We split domains by using male's samples as the source domain and female's samples as the target domain. The ratio between training and testing data is $0.5/0.5$. In the testing phase, we randomly selected instances from the source and target domains, with $n_s = 50$ and $n_t = 25$, respectively.
    
    \item \textbf{Breast Cancer Wisconsin (Original) dataset} contains 699 samples with 9 features to classify tumors as benign or malignant. We splited domain by using samples with $\text{``Clump Thickness''} \ge 3$ as the source domain and samples with $\text{``Clump Thickness''} < 3$ as the target domain. The ratio between training and testing data is $0.6/0.4$. In the testing phase, we randomly selected instances from the source and target domains, with $n_s = 50$ and $n_t = 25$, respectively.
    \item \textbf{CDC Diabetes Health Indicators dataset} contains 253,680 samples with 21 features to classify patients as $\text{``no diabetes''}$ and $\text{``prediabetes or diabetes''}$. We split domains by using female's samples as the source domain and male's samples as the target domain. The ratio between training and testing data is $0.6/0.4$. In the testing phase, we randomly selected instances from the source and target domains, with $n_s = 150$ and $n_t = 25$, respectively.
\end{itemize}

We conduct 120 independent runs for each experiment. During each run, all identified anomalies are evaluated, and their corresponding $p$-values are computed. 
Results are visualized in Fig. \ref{fig:real_exp_boxplots} and Fig. \ref{fig:real_exp_roc_curves}, which present the distributions of $p$-values and the ROC curves for both $\texttt{STAND-DA}$ and $\texttt{STAND-DA-oc}$, respectively. The ROC curves are derived by plotting the TPR against the FPR. To assess detection quality, we compare the area under the ROC curve (AUC) for each method, where a higher AUC reflects superior performance. As shown in Fig. \ref{fig:real_exp_boxplots}, the $p$-values from $\texttt{STAND-DA}$ are generally lower than those from $\texttt{STAND-DA-oc}$, suggesting that $\texttt{STAND-DA}$ achieves greater statistical power. Similarly, Fig. \ref{fig:real_exp_roc_curves} demonstrates that $\texttt{STAND-DA}$ consistently outperforms $\texttt{STAND-DA-oc}$.

\section{Conclusion}\label{sec:conclusion}

We introduced STAND-DA, a statistically rigorous framework for Autoencoder-based AD following RL-based DA. Leveraging the SI framework, STAND-DA provides valid $p$-values to rigorously control the FPR, even under complex DL architectures. To ensure practical scalability, we developed a GPU-accelerated implementation, enabling efficient application to real-world tasks. Experimental results validate both the effectiveness and efficiency of STAND-DA, establishing it as a reliable tool for statistically sound AD in domain-adapted settings.
Although this paper primarily focuses on AE-based AD and the RL-based DA method introduced in \cite{shen2018wasserstein}, our approach is readily applicable to a broader class of piecewise linear networks---i.e., models whose operations are characterized by linear inequalities or can be well-approximated by piecewise linear functions (many state-of-the-art DL models are, or can be,  approximated by piecewise linear networks). 
Our method can also be extended to the case where the mean squared error is used as the reconstruction error. In this setting, the set $\cZ_v$ can be characterized by a set of quadratic inequalities.
While our study centers on the AD task, extending the proposed method to other machine learning or computer vision tasks within the context of DL-based DA would represent a promising direction for future work.




\bibliographystyle{sn-mathphys}
\bibliography{ref}

\newpage

\begin{appendices}
\section{Appendix}

\subsection{Proof of Lemma \ref{lem:valid_p_value}}
\label{appx:proof_valid_selective_p}

We have 
\begin{align*}
    \bm \eta_j^\top {\rm vec}{X^s \choose X^t } \ \Bigg| \ \cC 
    \sim \mathbb{TN}\left( 
            \bm \eta_j^\top {{\rm vec}\left(M^s\right) \choose {\rm vec}\left(M^t\right)
            }, \bm \eta_j^\top \Sigma \bm \eta_j, \cZ
            \right),
\end{align*}

which is a truncated normal distribution with mean $\bm \eta_j^\top {{\rm vec}\left(M^s\right) \choose {\rm vec}\left(M^t\right)}$, variance $\bm \eta_j^\top \Sigma \bm \eta_j$, in which $\Sigma = \begin{pmatrix}
	\Sigma^s & 0 \\ 
	0 & \Sigma^t
\end{pmatrix}$, and the truncation region $\cZ$ described in \S\ref{subsec:truncation_region}. Therefore, under null hypothesis,

\begin{align*}
    p^{\rm selective}_j \ \Big| \ \cC   
    \sim \text{Unif}(0,1)
\end{align*}
Thus, $\mathbb{P}_{\rm H_{0, j}} \left( p^{\rm selective}_j \leq \alpha \ \Big| \ \cC 
\right) = \alpha, \forall \alpha \in [0,1]$.

Let us denote by $\cL$ the conditioning event of Eq. (\ref{eq:conditional_distribution}), defined as:
\begin{align*}
    \cL = \left\{ \cO_{X^s, X^t}
    =
    \cO_{\rm obs}, ~
    \cS_{X^s, X^t}
    =
    \cS_{\rm obs} \right\}
\end{align*}

Next, we have
\begin{align*}
    & \mathbb{P}_{\rm H_{0, j}} \left( p^{\rm selective}_j \leq \alpha \ \Big| \ \cL
\right)\\
    &= \int \mathbb{P}_{H_{0,j}} \left( p_j^{\rm selective} \leq \alpha \, \middle| \, \cC \right) \mathbb{P}_{H_{0,j}} \left( \cQ_{X^s, X^t} = \cQ_{\rm obs} \, \middle| \, \cL \right) \, d \cQ_{\rm obs} \\ 
    &= \int \alpha \, \mathbb{P}_{H_{0,j}} \left( \cQ_{X^s, X^t} = \cQ_{\rm obs} \, \middle| \, \cL \right) \, d\cQ_{\rm obs} \\ 
    &= \alpha \int \mathbb{P}_{H_{0,j}} \left( \cQ_{X^s, X^t} = \cQ_{\rm obs} \, \middle| \, \cL \right) \, d\cQ_{\rm obs} \\ 
    &= \alpha. 
\end{align*}

Finally, we obtain the result in Lemma \ref{lem:valid_p_value} as follows:
\begin{align*}
    \mathbb{P}_{\rm H_{0, j}} \left( p^{\rm selective}_j \leq \alpha
    \right) 
    &= \sum_{\cO_{\rm obs}} \sum_{\cS_{\rm obs}} \mathbb{P}_{\rm H_{0,j}} \left( p_j^{\rm selective} \leq \alpha \, \middle| \, \cL \right) \mathbb{P}_{\rm H_{0,j}} \left( \cL \right)\\
    &= \sum_{\cO_{\rm obs}} \sum_{\cS_{\rm obs}} \alpha \, \mathbb{P}_{\rm H_{0,j}} \left( \cO_{X^s, X^t}
    =
    \cO_{\rm obs}, ~
    \cS_{X^s, X^t}
    =
    \cS_{\rm obs} \right) \\
    &= \alpha \sum_{\cO_{\rm obs}} \sum_{\cS_{\rm obs}} \, \mathbb{P}_{\rm H_{0,j}} \left( \cO_{X^s, X^t}
    =
    \cO_{\rm obs}, ~
    \cS_{X^s, X^t}
    =
    \cS_{\rm obs} \right) \\
    &= \alpha.
\end{align*}

\subsection{Proof of Lemma \ref{lem:data_line}}
\label{appx:proof_data_line}

Base on the third condition in (\ref{eq:condition_event_set}), we have
\begin{align*}
    \cQ_{X^s, X^t} &= \cQ_{\rm obs} \\
    \Leftrightarrow \left( I_{n_s + n_t} - \bm b \bm \eta_j^\top \right) {\rm vec} { X^s \choose X^t } &= \cQ_{\rm obs} \\
    \Leftrightarrow {\rm vec} {X^s \choose X^t} &= \cQ_{\rm obs} + \bm b \bm \eta_j^\top {\rm vec} {X^s \choose X^t}.
\end{align*}
By defining $\bm a= \cQ_{\rm obs}, z = \bm \eta_j^\top {\rm vec}{X^s \choose X^t}$, and incorporating the third condition of (\ref{eq:condition_event_set}), we obtain Lemma \ref{lem:data_line}.


\subsection{Proof of Lemma \ref{lem:ad_algorithm}
\label{appx:proof_ad_algorithm}}
Let us define the input data, the data after the feature extraction and the reconstructed data as follows:
\begin{align*}
    &X(z) \in \mathbb{R}^{(n_s + n_t) \times d}, 
    \\
    &\tilde{X}(z) = f_{\rm extractor}(X(z)) = \tilde{A} + \tilde{B}z \in \mathbb{R}^{(n_s + n_t) \times d'}, \\ 
    &\hat{X}(z) = AE(\tilde{X}(z)) = \hat{A} + \hat{B}z \in \mathbb{R}^{(n_s + n_t) \times d'}.
\end{align*}
where $d'$ is the dimension of the data after the feature extraction, $f_{\rm extractor}$ and $AE$ are the feature extraction and autoencoder, respectively. Both $\tilde{X}(z) = \tilde{A} + \tilde{B}z$ and $\hat{X}(z) = \hat{A} + \hat{B}z$ can be recursively achieved through the process described in Lemma \ref{lem:activation_function}. The AE-based AD algorithm using $\ell_1$ reconstruction error can be described as follows: 


\begin{enumerate}
    \item Let's denote by $\mathcal{R}(z)$ is the set of $\ell_1$ reconstruction errors:
    \begin{align*}
    {\mathcal{R}}(z) &= \left\{ \|\tilde{X}_i(z) - \hat{X}_i(z)\|^1_1, \forall i \in [n_s+n_t] \right\} \\
    &= \left\{ \sum_{j \in [d']} |\tilde{X}_{i,j}(z) - \hat{X}_{i,j}(z)|, \forall i \in [n_s+n_t] \right\}
    \end{align*}
    The sign subtractions in this step can be represented by the following sets:
    \[
    \cI^{a_{i}}_{\bm a + \bm b z} = \left\{{\rm sign}(\tilde{X}_{i, j}(z) - \hat{X}_{i, j}(z)), \forall j \in [d'] \right\}, \quad \forall i \in [n_s + n_t].
    \]

    \item Using the $k^{th}$ percentile as the anomaly threshold, i.e., the top $(100-k)\%$ samples have highest reconstruction error are considered anomalies. Let's denote by $\cR_k(z)$ is value at the $k^{th}$ percentile of the set $\cR(z)$, the percentile event in this step can be represented by the following sets:
    \begin{align*}
        \cI^b_{\bm a + \bm b z} = \left\{ i \in [n_s+n_t]: \cR_i(z) \geq \cR_k(z) \right\}, \\
        \cI^c_{\bm a + \bm b z} = \left\{ j \in [n_s+n_t]: \cR_j(z) < \cR_k(z) \right\},
    \end{align*}
\end{enumerate}
The condition on the entire AE-based AD process can be presented as: 
\begin{align}
    \cI_v = \left\{z \in \mathbb{R} ~ \Bigg| ~  
    \begin{array}{l}
    \cI^{a_i}_{\bm a + \bm b z} = \cI^{a_i}_v, \forall i \in [n_s + n_t], \\ \\
    \cI^{b}_{\bm a + \bm b z} = \cI^{b}_v, \\ \\
    \cI^{c}_{\bm a + \bm b z} = \cI^{c}_v
    \end{array}
    \right\}
    \label{condition_Iv}
\end{align}

For any data point $X(z)$, if it satisfies all the conditions in \eqref{condition_Iv}, then $\cI_{\bm a + \bm bz} = \cI_v$. 

\vspace{5pt}

The first condition in \eqref{condition_Iv} can be represented as a set of linear inequalities: 
\begin{align*}
    &\left\{ \cI^{a_i}_{\bm a + \bm b z} = \cI^{a_i}_v, \quad \forall i \in [n_s + n_t] \right\}, 
    \\
    \Leftrightarrow &\left\{ \cI^{a_{i,j}}_{\bm a+\bm b z} = \cI^{a_{i,j}}_v, \quad \forall i \in [n_s + n_t], \forall j \in [d'] \right\}, 
    \\
    \Leftrightarrow &\left\{ {\rm sign}(\tilde{X}_{i, j}(z) - \hat{X}_{i, j}(z)) = \cI^{a_{i,j}}_v, \quad \forall i \in [n_s + n_t], \forall j \in [d'] \right\}, 
    \\
    \Leftrightarrow &\left\{\cI^{a_{i,j}}_v * (\tilde{X}_{i, j}(z) - \hat{X}_{i, j}(z)) \geq 0, \quad \forall i \in [n_s + n_t], \forall j \in [d']\right\},
    \\
    \Leftrightarrow &\left\{\cI^{a_{i,j}}_v * \left[(\tilde{A}_{i, j}-\hat{A}_{i, j}) + (\tilde{B}_{i, j} - \hat{B}_{i, j})z \right] \geq 0, \quad \forall i \in [n_s + n_t], \forall j \in [d']\right\},
    \\
    \Leftrightarrow &\left\{ \bm r^a z \leq \bm t^a \right\}.
\end{align*}

The second condition in \eqref{condition_Iv} can also be represented as a set of linear inequalities:

\begin{align*}
    &\left\{ \cI^b_{\bm a+ \bm b z} = \cI^b_{v} \right\}, 
    \\
    \Leftrightarrow &\left\{ \cR_i(z) \geq \cR_k(z), \quad \forall i \in \cI^b_{v} \right\}, 
    \\
    \Leftrightarrow &\left\{ \sum_{j \in [d']} |\tilde{X}_{i,j}(z) - \hat{X}_{i,j}(z)| \geq \sum_{j \in [d']} |\tilde{X}_{k,j}(z) - \hat{X}_{k,j}(z)|, \quad \forall i \in \cI^b_v \right\}, 
    \\ 
    \Leftrightarrow &\left\{ \sum_{j \in [d']} \cI^{a_{i, j}}_{\bm a+\bm b z}\left(\tilde{X}_{i,j}(z) - \hat{X}_{i,j}(z)\right) \geq \sum_{j \in [d']} \cI^{a_{k, j}}_{\bm a+\bm b z}\left(\tilde{X}_{k,j}(z) - \hat{X}_{k,j}(z)\right), \quad \forall i \in \cI^b_v \right\}
    \\
    \Leftrightarrow &\left\{ 
    \begin{array}{l}
         \sum_{j \in [d']} \cI^{a_{i, j}}_{\bm a+\bm b z} \left[(\tilde{A}_{i, j}-\hat{A}_{i, j}) + (\tilde{B}_{i, j} - \hat{B}_{i, j})z \right] \geq  
         \\\\
         \sum_{j \in [d']} \cI^{a_{k, j}}_{\bm a+\bm b z} \left[(\tilde{A}_{k, j}-\hat{A}_{k, j}) + (\tilde{B}_{k, j} - \hat{B}_{k, j})z \right], \quad \forall i \in \cI^b_v
    \end{array}
    \right\},
    \\
    \Leftrightarrow &\left\{
    \begin{array}{l}
     - \left[\sum_{j \in [d']} \cI^{a_{i, j}}_{\bm a+\bm b z} * (\tilde{B}_{i, j}-\hat{B}_{i, j}) - \cI^{a_{k, j}}_{\bm a+\bm b z} *(\tilde{B}_{k, j}-\hat{B}_{k, j})\right]z \leq  
     \\\\
    \quad \;\,\sum_{j \in [d']} \cI^{a_{i, j}}_{\bm a+\bm b z} * (\tilde{A}_{i, j}-\hat{A}_{i, j}) - \cI^{a_{k, j}}_{\bm a+\bm b z} *(\tilde{A}_{k, j}-\hat{A}_{k, j}), \quad \forall i \in \cI^b_v
    \end{array}
    \right\},
    \\
    \Leftrightarrow &\left\{ \bm r^b z \leq \bm t^b \right\}.
\end{align*}

The final condition in \eqref{condition_Iv} can be characterized with the similar approach as the second condition. This condition also is represented as a set of linear inequalities: 
\begin{align*}
    &\left\{ \cI^c_{\bm a+ \bm b z} = \cI^c_{v} \right\},
    \\
    \Leftrightarrow &\left\{ \bm r^c z \leq \bm t^c\right\}.
\end{align*}

Finally, by defining 
\begin{align*} 
    \bm r = \begin{pmatrix}
        \bm r^a \\ \bm r^b \\ \bm r^c
    \end{pmatrix} \text{and} ~~ \bm t = \begin{pmatrix}
        \bm t^a \\ \bm t^b \\ \bm t^c
    \end{pmatrix},
\end{align*}
we obtain the result in Lemma \ref{lem:ad_algorithm}.

\subsection{Proof of Lemma \ref{lem:test_statistic_calculation}}
\label{appx:proof_test_statistic_calculation}
For any $X(z)$, the sign of subtractions in the computation of the test statistic in \eqref{eq:test_statistic} is defined as:
\begin{align*}
    \cS_{\bm a + \bm b z} &= \left\{{\rm sign}\left(X^t_{j,k} - \bar{X}^t_{j,k}\right), \quad \forall k \in [d] \right\}\\
    &= \left\{{\rm sign}
    \left(
    \left[
    \begin{pmatrix}
        \bm 0^{s} \\
        \bm e^{t}_j
    \end{pmatrix}^\top 
        \begin{pmatrix}
            X^s \\
            X^t
        \end{pmatrix}
        - \frac{1}{n_t - |\cO|} \sum \limits_{ \ell \in [n_t] \setminus \mathcal{O}}
        \begin{pmatrix}
        \bm 0^{s} \\
        \bm e^{t}_\ell
    \end{pmatrix}^\top 
        \begin{pmatrix}
            X^s \\
            X^t
        \end{pmatrix}
    \right]_{k}
    \right)
    , \quad \forall k \in [d]\right\},
\end{align*}
where $\bm e^t_j \in \mathbb{R}^{n_t}$ is a vector with 1 at the $j$-th position and 0 elsewhere. 

The conditional event $\left\{\cS_{\bm a+ \bm b z} = \cS_{\rm obs}\right\}$, w.r.t. $\begin{pmatrix}
    X^s \\
    X^t
\end{pmatrix} = A + Bz$,  
is characterized as:
\begin{align*}
    &\bm s_j^\top \circ \left(\begin{pmatrix}
        \bm 0^{s} \\
        \bm e^{t}_j
    \end{pmatrix}^\top \begin{pmatrix}
        A + Bz
    \end{pmatrix} - \frac{1}{n_t - |\cO|} \sum \limits_{ \ell \in [n_t] \setminus \mathcal{O}} \begin{pmatrix}
        \bm 0^{s} \\
        \bm e^{t}_l
    \end{pmatrix}^\top \begin{pmatrix}
        A + Bz
    \end{pmatrix}
    \right) \geq 0 \\
    \Leftrightarrow
    &- \bm s_j^\top \circ \left(
    \begin{pmatrix}
        \bm 0^{s} \\
        \bm e^{t}_j
    \end{pmatrix} - \frac{1}{n_t - |\cO|} \sum \limits_{ \ell \in [n_t] \setminus \mathcal{O}} 
    \begin{pmatrix}
        \bm 0^{s} \\
        \bm e^{t}_l
    \end{pmatrix}\right)^\top B z
    \leq
    \bm s_j^\top \circ \left(\begin{pmatrix}
        \bm 0^{s} \\
        \bm e^{t}_j
    \end{pmatrix} - 
    \frac{1}{n_t - |\cO|} \sum \limits_{ \ell \in [n_t] \setminus \mathcal{O}}
    \begin{pmatrix}
        \bm 0^{s} \\
        \bm e^{t}_l
    \end{pmatrix}\right)^\top A,
\end{align*}
where $\bm s_j$ is the sign vectors corresponding to $\cS_{\rm obs}$ defined in \eqref{eq:etaj}.
Finally, by defining 
\begin{align*}
    \bm w &=
    - \bm s_j^\top \circ \left(
    \begin{pmatrix}
        \bm 0^{s} \\
        \bm e^{t}_j
    \end{pmatrix} - \frac{1}{n_t - |\cO|} \sum \limits_{ \ell \in [n_t] \setminus \mathcal{O}} 
    \begin{pmatrix}
        \bm 0^{s} \\
        \bm e^{t}_l
    \end{pmatrix}\right)^\top B \\
    \text{and} \quad \bm o &= \bm s_j^\top \circ \left(
    \begin{pmatrix}
        \bm 0^{s} \\
        \bm e^{t}_j
    \end{pmatrix} - \frac{1}{n_t - |\cO|} \sum \limits_{ \ell \in [n_t] \setminus \mathcal{O}} 
    \begin{pmatrix}
        \bm 0^{s} \\
        \bm e^{t}_l
    \end{pmatrix}\right)^\top A,
\end{align*}
we obtain the result in Lemma \ref{lem:test_statistic_calculation}. 
 
\end{appendices}

\end{document}